\pdfoutput=1
\documentclass[11pt,a4paper]{article}
\usepackage[margin=1in]{geometry}
\usepackage[utf8]{inputenc}
\usepackage[T1]{fontenc}
\usepackage{amsmath,amssymb,amsthm}
\usepackage{booktabs}
\usepackage{microtype}
\usepackage[colorlinks=true,linkcolor=blue,citecolor=blue,urlcolor=blue]{hyperref}
\usepackage{parskip}
\usepackage{array}

\newtheorem{theorem}{Theorem}[section]
\newtheorem{proposition}[theorem]{Proposition}
\newtheorem{lemma}[theorem]{Lemma}
\newtheorem{corollary}[theorem]{Corollary}
\newtheorem{definition}[theorem]{Definition}
\newtheorem{example}[theorem]{Example}
\newtheorem{remark}[theorem]{Remark}

\DeclareMathOperator{\Dom}{Dom}
\DeclareMathOperator{\Safe}{Safe}
\DeclareMathOperator{\Reach}{Reach}
\DeclareMathOperator{\Next}{Next}
\newcommand{\Pspec}{\preceq_{\mathrm{spec}}}
\newcommand{\PICT}{\mathcal{P}_{ICT}}
\newcommand{\Dcomp}{D_{\mathrm{comp}}}
\newcommand{\DcompHTF}{D_{\mathrm{comp}}^{\mathrm{HTF}}}
\newcommand{\Pifull}{\Pi_{\mathrm{full}}}

\title{Determinization in Structure Theories:\\
A Unified Framework via Closure, Comparability,\\
and Joint Admissibility}

\author{Hai Hai Fu%
\thanks{A preliminary version is available at SSRN~6582699.
\url{https://ssrn.com/abstract=6582699}}}

\date{Version v2.16.4 \quad April 2026}

\begin{document}
\maketitle

\begin{abstract}
Motivated by structural failures in LLM-assisted reasoning and decision
systems, we develop a formal framework for constructing canonical
interpretations from plural structure theories. A \emph{structure theory}
is a triple $T = (\Sigma, A, I)$ of a signature, axioms, and inference
policy. Its \emph{admissible interpretation family} $\mathcal{P}_T$
collects all globally consistent assignments of structural conclusions to
positions. In LLM-assisted reasoning, hallucination can be viewed as
unsupported canonicalization---the system emits a determinate answer when
the underlying admissible set has not collapsed to a singleton. Our
framework formalizes when canonicalization is licensed.

Not all structure theories are deterministic. We classify non-determinism
into \emph{structural plurality} (Type~S) and \emph{epistemic plurality}
(Type~E), instantiated by Wyckoff and ICT respectively, relative to the
intrinsically deterministic Chan theory.

Our framework distinguishes three increasingly strong notions of
``canonical'': \emph{closure stabilization} (per-seed convergence to some
fixed point, property~(4')); \emph{global completion} (all seeds converge
to a single fixed point, property~(4)); and \emph{determinization} in the
AC-6 sense (a unique admissible interpretation per input). The
construction lemma of \S5.3 (Theorem~6b) records sufficient structural
conditions under which a canonicalization mechanism (completion operator
$\Dcomp$ for Type~E, selector PhaseClassify for Type~S-strong) can be
constructed at closure-stabilization strength. Theorem~6b is a
construction lemma; its substantive content is unpacked in Corollary~6b'
(the Class~C classification of pure $I^{\mathrm{core}}$-computable
completion under R1+R2+R3 plus rule Soundness, proved via Lemma~C and
the Monotonic Exhaustion Lemma in \S8.2). For Type~E theories (ICT),
upgrading from~(4') to~(4) and AC-6 requires an additional global
confluence result; for ICT this is the open question OQ-GC-1, and we are
explicit that \emph{the present paper does not establish ICT
determinization at AC-6 strength}. For Type~S-strong theories (Wyckoff),
AC-6 is achieved directly via canonical selection (PhaseClassify).

Two canonicalization forms exist: \emph{operator-based completion} and
\emph{selector-based construction}. These are not the same kind of
mathematical object: completion is operator-theoretic; selection is
selector-based. Completion requires closure, comparability, and a
compatible-extension condition (C5-E); selection requires theory-intrinsic
comparability (N4) and global compatibility of selected choices (C5-S).
A minimal counterexample $T_{ce}$ shows that closure alone does not
suffice: theory-intrinsic comparability (N4) is necessary.

We further show that completion extends to multi-timeframe hierarchies via
a staged operator $\DcompHTF = D_L \circ D_H$, and establish that this
operator is \textbf{structurally non-commutative}: within the specified
two-operator raw staged architecture, HTF-first ordering is the only
admissibility-preserving linearization.

Section~8 presents a classification of canonicalization mechanisms within
the MST framework. We define theory-intrinsic operations independently of
the mechanism taxonomy (definable from $T$ and automorphism-invariant);
Theorem~U3, conditional on the realization-coverage assumption tracked as
OQ-Realization and on per-realization branch hypotheses (R1+R2+R3 +
Soundness + elementary $F$-step for (3a); theory-intrinsic
$\preceq_{\mathrm{sel}}$ + C5-S for (3b); R1+R2+R3 + (G2) + (G3) +
Soundness for (3c)), classifies every theory-intrinsic mechanism realized
within the (3a)/(3b)/(3c) grammar into Class~C (closure) or Class~S
(selection). The classification is at closure-stabilization
strength~(4'); whether a Class~C operator additionally achieves
property~(4) or AC-6 is instance-specific. The section also introduces
canonicalization as a partial pseudofunctor (Det).

For Type~S, the relationship to selection is more nuanced than a strict
dichotomy. We isolate \textbf{Type~S-strong} as the subclass where
$\mathcal{P}_T$ contains two admissible interpretations with no common
upper bound under $\Pspec$ (\S4.2). Only for Type~S-strong instances is
property~(4) (global completion) provably impossible; the weaker~(4')
remains attainable. Wyckoff is Type~S-strong, so the selection
construction is justified for $T_{Wy}$; the generic Type~S $\to$
selection inference is conditional on this strengthening and is left open
(OQ-TypeS-Imp).

Phase~2 (\S10) characterizes canonicalization under multiple operators
with precedence constraints (summary statement; full proofs in companion
artifacts). Feasibility is equivalent to the vanishing of the safety gap
$\Phi^*$. Under pairwise local compatibility
($\mathrm{PLC}_{\mathcal{S}}$), the existence conditions (C2) and (C2')
collapse. The compatibility-condition strength ordering is
$\mathrm{GCC} \Rightarrow \mathrm{RUC}_{\mathcal{S}} \Rightarrow
\mathrm{PLC}_{\mathcal{S}}$ (strict in the seed-dependent setting,
collapsing to $\mathrm{GCC} \Rightarrow \mathrm{RUC} \equiv \mathrm{PLC}$
seed-independently); detailed proofs are recorded in companion artifacts
and not reproduced in this paper.

\medskip
\noindent\textbf{Core principle.} \emph{Completion-based canonicalization
requires closure, comparability, and compatible extension; selection-based
canonicalization requires theory-intrinsic comparability and global
compatibility of selected choices. Not all canonicalization is
operator-theoretic. Under hierarchical interaction, canonicalization
operators form a structurally non-commutative system.}
\end{abstract}

\tableofcontents
\newpage

\section{Introduction}

This paper addresses a single question:

\medskip
\noindent\textbf{When can a structure theory be made canonical?}
\medskip

The motivating AI setting is LLM-assisted structural decision systems: a
language model proposes candidate interpretations, while a symbolic
structure theory specifies admissible interpretations. Canonicalization is
the problem of deciding when the system is licensed to output a unique
interpretation rather than a set of admissible alternatives. Hallucination,
in this view, is unsupported canonicalization---the system emits a
determinate answer when the underlying admissible set has not collapsed to
a singleton. We formalize the structural conditions under which
canonicalization is licensed.

Most market structure frameworks are \emph{not} deterministic: multiple
valid interpretations coexist at the same position. We formalize this
question and characterize sufficient structural conditions under which a
canonical reading can be constructed.

We formalize using structure theories $T = (\Sigma, A, I)$, and introduce
\emph{canonicalization} as the construction of a canonical extension
$T^{\mathrm{canon}}$ from a plural base $T$. Our framework operates at
three strengths:

\begin{itemize}
\item \textbf{Closure stabilization} (property~(4'), \S4.1): each
  admissible seed converges under iteration of an operator to \emph{some}
  fixed point. This is the strength established by Theorem~6b's completion
  branch and by Corollary~6b'.
\item \textbf{Global completion} (property~(4)): all seeds converge to a
  \emph{single} fixed point. This requires global confluence of the
  operator's ARS, which is theory-specific and not implied by~(4'). For
  ICT this is open (OQ-GC-1).
\item \textbf{Determinization} (AC-6): the unique fixed point is the sole
  admissible interpretation. This is the strongest claim and is achieved
  directly by canonical selection (e.g., PhaseClassify for Wyckoff). For
  ICT, AC-6 is conditional on OQ-GC-1.
\end{itemize}

We are explicit throughout the paper about which claim applies at each
strength, and in particular about the fact that \textbf{the present paper
establishes ICT closure stabilization but does not establish ICT
determinization at AC-6 strength}.

Our main contributions are:

\begin{enumerate}
\item \textbf{Taxonomy} (\S3, \S4.2): two-type classification of
  non-determinism (Type~S vs.\ Type~E), with the \emph{Type~S-strong}
  refinement (no common upper bound) for which selection is provably the
  unique mechanism reaching AC-6.
\item \textbf{Construction lemma} (\S5.3, Theorem~6b): sufficient
  structural conditions (N0--N4 + C5, with mechanism-specific variants
  C5-E and C5-S) under which the completion operator $\Dcomp$ or the
  selector PhaseClassify can be constructed. Theorem~6b is a construction
  lemma: given the conditions, the mechanism is built. Its substantive
  depth is unpacked by Corollary~6b' (\S5.3).
\item \textbf{Classification of $I^{\mathrm{core}}$-computable completion}
  (Corollary~6b', \S5.3 + Lemma~C and Monotonic Exhaustion Lemma, \S8.2):
  under R1+R2+R3 \emph{plus rule Soundness relative to $A$}, every
  (3a)-computable completion operator (interpreted as the saturated
  closure $F^\omega$ under elementary one-rule firing) is automatically
  monotone, extensive, admissibility-preserving, and converges per seed
  to a fixed point---hence is Class~C at closure-stabilization strength.
\item \textbf{Dual mechanism} (\S4.1, \S4.2): two fundamentally different
  canonicalization forms---operator-based completion (Type~E) vs.\
  selector-based construction (Type~S-strong, axiom-induced order). The
  two constructions yield different mathematical objects.
\item \textbf{Counterexample} $T_{ce}$ (\S5.2): a minimal theory
  satisfying N0+N1+N3 but failing both N4-sel and N4-comp; demonstrates
  that theory-intrinsic comparability (N4) is necessary.
\item \textbf{Multi-timeframe extension} (\S6.2): staged operator
  $\DcompHTF = D_L \circ D_H$ on $\PICT^{\mathrm{multi}}$, with explicit
  verification of \S4.1 properties at~(4') strength.
\item \textbf{Structural non-commutativity} (\S7): a concrete witness $P$
  such that $(D_L \circ D_H)(P) \in \PICT^{\mathrm{multi}}$ while
  $(D_H \circ D_L)(P) \notin \PICT^{\mathrm{multi}}$, with ordering
  uniqueness scoped to the two-operator raw staged architecture.
\item \textbf{Mechanism classification} (\S8): asymmetric reduction (R2*,
  parts~i and~ii); independent definition of ``theory-intrinsic'' not
  predicated on (3a)/(3b)/(3c); Theorem~U3, conditional on
  realization-coverage (OQ-Realization) and on per-realization branch
  hypotheses, classifies every theory-intrinsic mechanism realized within
  (3a)/(3b)/(3c) into Class~C or Class~S; Det as a partial pseudofunctor
  (\S8.4).
\item \textbf{Phase~2 (multi-operator canonicalization)} (\S10): summary
  of existence conditions ($\Phi^* = 0 \Leftrightarrow$ admissible
  linearization), collapse under
  $\mathrm{PLC}_{\mathcal{S}}$ (C2 $\equiv$ C2' $\equiv$ existence), and
  the compatibility-strength hierarchy
  $\mathrm{GCC} \Rightarrow \mathrm{RUC}_{\mathcal{S}} \Rightarrow
  \mathrm{PLC}_{\mathcal{S}}$ (summary statement; full proofs are
  recorded in companion artifacts and not reproduced here).
\end{enumerate}

\noindent\textit{Scope note.} This paper unifies several lines of MST
development under a single framework. Each of contributions~1--3, 6--7,
8, and~9 could be developed in a focused follow-up paper, and we
anticipate such future splits as the framework matures. The \S11
changelog tracks the v2.15.1 $\to$ v2.16 $\to$ v2.16.2 $\to$ v2.16.3
$\to$ v2.16.4 errata revisions.

\section{Framework}

\subsection{Structure Theories}

\begin{definition}
A \emph{structure theory} is $T = (\Sigma, A, I)$ where $\Sigma$ is a
signature, $A$ axioms, and $I = I^{\mathrm{core}} \cup I^{\mathrm{ext}}$
an inference policy (this paper concerns only $I^{\mathrm{core}}$).
\end{definition}

\noindent\textbf{Input tuple.} Theories are evaluated on inputs
$(N, \Theta, \Pi)$ where $N$ is an observed market realization, $\Theta$
is the set of theory-relevant parameters, and $\Pi$ is a structural
evidence assignment. Domain objects: $R$ = set of regions/instruments,
$\mathsf{Time}$ = temporal index set, $\mathcal{C}$ = structural
conclusion vocabulary.

\begin{definition}[Finite evaluation domain]
Throughout we assume:
(a) $\mathcal{C}$ is finite;
(b) $\Dom(N,\Theta) \subseteq R \times \mathsf{Time}$ is finite for every
$(N,\Theta)$. Hence
$|\mathcal{P}_T(N,\Theta,\Pi)| \leq (2^{|\mathcal{C}|})^{|\Dom(N,\Theta)|}
< \infty$.
\end{definition}

\noindent\textbf{Definitional constraints on $I^{\mathrm{core}}$.}
$I^{\mathrm{core}}$ is restricted to \emph{positive, non-retractive}
inference:

\begin{itemize}
\item \textbf{(R1) Positive rule form}: every rule only adds structural
  conclusions; no rule removes conclusions.

\item \textbf{(R2) No negation-as-failure in $I^{\mathrm{core}}$}: every
  rule has a premise that is a positive conjunction of conditions---each
  condition asserts the \emph{presence} of a pattern or comparison on
  observed data, or the \emph{presence} of an already-derived predicate,
  but never the \emph{absence} of a derived predicate.

  \emph{Scope}: R2 constrains $I^{\mathrm{core}}$. The determinization
  layer (\S4) may employ conflict-resolution mechanisms (e.g., M2/M2-HTF
  guards, \S6.2) that are not $I^{\mathrm{core}}$ rules and are not
  constrained by R2.

\item \textbf{(R3) No retraction}: no rule withdraws or invalidates
  previously derived conclusions.
\end{itemize}

These constraints are definitional and \emph{syntactic}---they constrain
the shape of rules. They are independent of \emph{semantic} soundness
(consistency of rule conclusions with the axioms~$A$); see \S5.3 and the
soundness Remark following Corollary~6b'. Conflict resolution among
admissible conclusions is not the responsibility of $I^{\mathrm{core}}$;
it is handled by the determinization mechanism (Class~C completion or
Class~S selection). This separation is the architectural basis for the
mechanism classification in \S8.

\begin{definition}
The \emph{admissible interpretation family} $\mathcal{P}_T(N,\Theta,\Pi)$
is the set of all assignments
$P: \Dom(N,\Theta) \to 2^{\mathcal{C}}$ satisfying all axioms~$A$ and
consistent with $I^{\mathrm{core}}$.
\end{definition}

\begin{definition}[Specification order]
For $P, Q \in \mathcal{P}_T$, write $P \Pspec Q$ if
$P(r,t) \subseteq Q(r,t)$ for all $(r,t) \in \Dom(N,\Theta)$. This is a
partial order on $\mathcal{P}_T$.
\end{definition}

\begin{definition}[AC-6 and AC-6']
$T$ is \emph{deterministic} (AC-6) if $|\mathcal{P}_T(N,\Theta,\Pi)| = 1$
for all inputs. $T$ is \emph{plural} (AC-6') if
$|\mathcal{P}_T(N,\Theta,\Pi)| \geq 2$ for some input.
\end{definition}

\section{Taxonomy of Non-Determinism}

\subsection{Formal Predicates}

\begin{definition}[Type~E --- Epistemic Plurality]
$T$ is \emph{Type~E} if its plurality vanishes under structural evidence
refinement: for every input with $|\mathcal{P}_T| > 1$, there exists a
refined $\Pi' \supseteq \Pi$ (with $N$, $\Theta$ fixed) such that
$|\mathcal{P}_T(N,\Theta,\Pi')| = 1$.
\end{definition}

\begin{definition}[$\Pifull$]
A \emph{maximally informative evidence assignment} assigns a definite
truth value to every theory-relevant structural predicate in
$\mathcal{C}$.
\end{definition}

\begin{definition}[Type~S --- Structural Plurality]
$T$ is \emph{Type~S} if its plurality persists at $\Pifull$: there
exists $(N,\Theta,\Pifull)$ with
$|\mathcal{P}_T(N,\Theta,\Pifull)| > 1$.
\end{definition}

\begin{definition}[Type~S-strong --- No Common Upper Bound]
$T$ is \emph{Type~S-strong} if $T$ is Type~S and there exists an input
$(N,\Theta,\Pifull)$ at which two distinct admissible interpretations
$P_1, P_2 \in \mathcal{P}_T(N,\Theta,\Pifull)$ have \emph{no common
upper bound} in $(\mathcal{P}_T, \Pspec)$.
\end{definition}

\begin{remark}[Type~S vs.\ Type~S-strong]
Type~S is the categorical predicate; Type~S-strong is the substructural
strengthening. A Type~S theory in which $P_1 \Pspec P_2$ at $\Pifull$ is
Type~S but not Type~S-strong; the impossibility result of \S4.2 does not
apply in such cases. The canonical Type~S example $T_{Wy}$ is
Type~S-strong (axiomatic phase incomparability via [Wy-Sem-1]).
\end{remark}

\begin{remark}[Monotonicity and exhaustiveness]
Under evidence-refinement monotonicity ($\Pi \subseteq \Pi' \Rightarrow
\mathcal{P}_T(N,\Theta,\Pi') \subseteq \mathcal{P}_T(N,\Theta,\Pi)$):
Type~E $\iff$ unique at $\Pifull$ for all $(N,\Theta)$;
Type~S $\iff$ not unique at $\Pifull$ for some $(N,\Theta)$.
The two types are exhaustive for plural theories; canonical instances
(ICT = Type~E, Wyckoff = Type~S-strong) are purely one type.
\end{remark}

\subsection{Examples}

\textbf{Type~S-strong: Wyckoff ($T_{Wy}$).} Axiom [Wy-Sem-1] makes
\texttt{in\_phase} relational; admissible phases at $\Pifull$ are
axiomatically incomparable, with no element of $\mathcal{P}_{Wy}$
extending two distinct phase assignments. Hence Type~S-strong.

\textbf{Type~E: ICT ($T_{ICT}$).} Plurality is epistemic: once additional
bars resolve the structure, exactly one completion survives.

\textbf{Intrinsically Deterministic: Chan ($T_{\mathrm{Chan}}$).} Chan is
deterministic by the unique-normal-form theorem for its term rewrite
system $\mathcal{R}_{\mathrm{Chan}}$.

\medskip\noindent\textit{Chan rewrite presentation (sketch).}
$\mathcal{R}_{\mathrm{Chan}}$ encodes the bi-directional constituency
relations of Chan theory, a hierarchical market structure framework with
multiple structural levels (elementary strokes, line segments, pivot
zones). The TRS is asserted to satisfy termination (rules strictly
decrease structural nesting depth) and local confluence (critical pairs
are joinable). By Newman's Lemma, $\mathcal{R}_{\mathrm{Chan}}$ is
confluent: every expression has a unique normal form, establishing AC-6
without external construction.

\medskip\noindent\textit{Open question (OQ-Chan-TRS).} Termination and
local confluence are stated here as the standard informal account in the
Chan literature. A self-contained formal proof is left as an open
question. The framework does not depend on resolving OQ-Chan-TRS for
Type~E or Type~S-strong results; Chan is used as a comparator, not a
proof dependency.

\medskip
\begin{center}
\begin{tabular}{lll}
\toprule
Theory & Type & Determinism \\
\midrule
Chan & Rewrite-presentable (pending OQ-Chan-TRS) & AC-6 (unique normal form) \\
ICT  & Type~E & AC-6' ((4') established; AC-6 conditional on OQ-GC-1) \\
Wyckoff & Type~S-strong & AC-6 (via PhaseClassify selection) \\
\bottomrule
\end{tabular}
\end{center}

\section{Determinization Mechanisms}

\subsection{Operator-Based Completion (Type~E)}

\begin{definition}[Determinization operator]
\label{def:det-op}
Let $\mathcal{D}_T \subseteq \mathcal{P}_T$ be a canonical closure
domain. A \emph{determinization operator} is a map
$D: \mathcal{D}_T \to \mathcal{D}_T$ satisfying:
\begin{enumerate}
\item \textbf{Extensivity}: $P \Pspec D(P)$ for all
  $P \in \mathcal{D}_T$.
\item \textbf{Idempotence}: $D(D(P)) = D(P)$ for all
  $P \in \mathcal{D}_T$.
\item \textbf{Admissibility preservation}: $D(P) \in \mathcal{D}_T$ for
  all $P \in \mathcal{D}_T$.
\item \textbf{Per-seed convergence (Property~(4'))}: for every
  $(N,\Theta,\Pi)$ and every seed
  $P \in \mathcal{D}_T(N,\Theta,\Pi)$, $\exists\, k \leq
  |\mathcal{D}_T(N,\Theta,\Pi)|$ and a fixed point
  $P^*(P) \in \mathcal{D}_T$ such that $D^k(P) = P^*(P)$ and
  $D(P^*(P)) = P^*(P)$.
\end{enumerate}
The stronger \textbf{Property~(4)} additionally requires all seeds
converge to the same fixed point.
$\mathrm{AC\text{-}6} \Rightarrow (4) \Rightarrow (4')$.
\end{definition}

\begin{remark}[Three levels of canonicalization]
\emph{Completion (weak)}: satisfies (1)--(3) and~(4'); closure
stabilization. \emph{Completion (strong)}: satisfies (1)--(3) and~(4);
global completion. \emph{Determinization} (AC-6): unique fixed point is
the sole admissible interpretation.
\end{remark}

\begin{definition}[Uniqueness-inducing operator]
$D: \mathcal{D}_T \to \mathcal{D}_T$ is \emph{uniqueness-inducing} if it
satisfies property~(4): there exists, for each instance $(N,\Theta,\Pi)$,
a single $P^*(N,\Theta,\Pi) \in \mathcal{D}_T(N,\Theta,\Pi)$ such that
$D^k(P) = P^*(N,\Theta,\Pi)$ for all seeds $P$ and some $k$ uniformly
bounded across seeds.
\end{definition}

An operator satisfying only~(4') (per-seed convergence) is \emph{not}
uniqueness-inducing in this sense---different seeds may stabilize at
distinct fixed points $P^*(P)$.

\textbf{ICT instance.} $\Dcomp(P) = \mathrm{Cl}_{A_{ICT},\Pi}(P)$.
Under M1+M2+M3, $\Dcomp$ satisfies all four properties on
$\mathcal{D}_T = \mathcal{P}_{ICT}^{\mathrm{closure}}$, with property~(4)
at strength~(4') only---the upgrade to global~(4) is OQ-GC-1.

\subsection{Selector-Based Construction (Type~S-strong)}

For Type~E theories, completion naturally applies. For Type~S-strong
theories---those with incomparable admissible interpretations at
$\Pifull$---completion at property~(4) strength is provably impossible,
and canonicalization at AC-6 strength must use a canonical selector. The
weaker property~(4') is not blocked.

\begin{proposition}[Completion at~(4)-strength impossible for Type~S-strong]
\label{prop:completion-impossible}
Let $T$ be Type~S-strong, witnessed at $\Pifull$ by
$P_1, P_2 \in \mathcal{P}_T(N,\Theta,\Pifull)$ with no common upper bound
in $(\mathcal{P}_T, \Pspec)$. Then for any extensive operator
$D: \mathcal{D}_T \to \mathcal{D}_T$ with $P_1, P_2 \in \mathcal{D}_T$,
property~(4) fails on this instance.
\end{proposition}

\begin{proof}
Suppose for contradiction that $D$ is extensive and property~(4) holds
with unique convergence point $P^*$.

\emph{Step~1 (ascending chain).} Extensivity gives
$P_1 \Pspec D(P_1) \Pspec D^2(P_1) \Pspec \cdots$. Since
$\mathcal{D}_T \subseteq \mathcal{P}_T$ is finite, the chain stabilizes:
$\exists\, k_1 \leq |\mathcal{D}_T|$ with
$D^{k_1}(P_1) = D^{k_1+1}(P_1)$.

\emph{Step~2 (stabilization equals $P^*$).} $D^{k_1}(P_1)$ is a fixed
point of $D$ in $\mathcal{D}_T$, hence a valid seed. By~(4), every seed
converges to $P^*$; the constant orbit from $D^{k_1}(P_1)$ converges to
itself, so $D^{k_1}(P_1) = P^*$. By transitivity, $P_1 \Pspec P^*$.

\emph{Step~3 (symmetric).} $P_2 \Pspec P^*$.

\emph{Step~4 (contradiction).} $P^*$ is a common upper bound of
$\{P_1, P_2\}$, contradicting Type~S-strong.

The weaker property~(4') is not ruled out: $P_1$ and $P_2$ may
individually converge to distinct fixed points $P_1^*, P_2^*$.
\qed
\end{proof}

\begin{remark}[Scope]
The proposition (a) requires Type~S-strong (not merely Type~S);
(b) excludes property~(4) but not~(4');
(c) is mechanism-agnostic except for extensivity. For $T_{Wy}$, all three
apply: $T_{Wy}$ is Type~S-strong, so any extensive $D$ on
$\mathcal{D}_{T_{Wy}}$ fails~(4); the canonical selector PhaseClassify
achieves AC-6 directly.
\end{remark}

\begin{remark}[Open question OQ-TypeS-Imp]
Type~S theories that are not Type~S-strong are not addressed by this
proposition. Whether such theories admit completion at~(4) strength
depends on the upper-bound structure of $\mathcal{P}_T$ and is left as
OQ-TypeS-Imp. The dichotomy ``Type~E $\leftrightarrow$ completion,
Type~S $\leftrightarrow$ selection'' is a sound implication only for
Type~E (at~(4') strength) and Type~S-strong (at AC-6 strength).
\end{remark}

\begin{remark}[ICT instance]
For $T_{ICT}$, $\Dcomp$ satisfies (1)--(3) and~(4') (per-seed
convergence). Whether all seeds converge to the same fixed
point---property~(4)---requires a global confluence result for
$A_{ICT}$ (OQ-GC-1, not established here). The proposition above
concerns Type~S-strong; $T_{ICT}$ (Type~E) is not in its scope.
\end{remark}

\begin{definition}[Canonical selector]
A function $\phi: (r,t,\Theta,\Pi) \mapsto c^* \in A_{r,t}$,
theory-intrinsically defined and globally compatible.
\end{definition}

\textbf{Wyckoff instance.}
$\mathrm{PhaseClassify}(r,t;\Theta,\Pi) =
\max_{\preceq_{\mathrm{sel}}} A_{r,t}$, where $\preceq_{\mathrm{sel}}$
is induced by event dominance (C4-Q2) and MEP ensures global
compatibility.

\medskip
\begin{center}
\begin{tabular}{lll}
\toprule
& Completion ($\Dcomp$) & Selection (PhaseClassify) \\
\midrule
Mathematical type & Operator on $\mathcal{D}_T$ & Selector on $A_{r,t}$ \\
Mechanism & Closure to fixed point & Max under total order \\
Extensivity required & Yes & No \\
Domain & $\mathcal{P}_{ICT}^{\mathrm{closure}}$ & Full (under MEP) \\
\bottomrule
\end{tabular}
\end{center}

\section{Sufficient Conditions for Canonicalization (Theorem~6b)}

\subsection{Conditions}

\begin{itemize}
\item \textbf{N0} (universal, non-emptiness):
  $\mathcal{P}_T(N,\Theta,\Pi) \neq \varnothing$ for every input.

\item \textbf{N1} (selector-specific, finiteness): $|A_{r,t}| < \infty$.

\item \textbf{N3} (operator-specific, bounded closure):
  $\exists\, k: \mathrm{Cl}^k(P) = \mathrm{Cl}^{k+1}(P)$ on the canonical
  seed domain (M3 ensures this; bound
  $|\mathcal{C}| \cdot |\Dom(N,\Theta)|$).

\item \textbf{N4} (unified comparability): $T$ admits a theory-intrinsic
  comparability-inducing structure $\mathcal{K}$ providing local
  comparability, unique selection or unique closure fixed point,
  derivability, and global compatibility.

  \emph{Refined treatment.} The structural definition---order tuple
  $\mathcal{K} = (D, \preceq_{\mathcal{K}}, \mathcal{A}, \mathcal{C})$
  with selector-/closure-/piecewise-compatible specializations---is given
  in \S8.3. Theorem~6b's branches use N4-comp (closure-compatible) and
  N4-sel (selector-compatible).

\item \textbf{C5} (joint admissibility, mechanism-specific). C5 is
  expressed as two distinct conditions according to mechanism type:

  \textbf{C5-S} (selection variant, \emph{global compatibility of
  selected choices}): let
  $c^*_{r,t} := \max_{\preceq_{\mathrm{sel}}} A_{r,t}$ be the selector
  output at $(r,t)$. Then the selected assignment
  $$P^*: (r,t) \mapsto \{c^*_{r,t}\}$$
  extends to an element of $\mathcal{P}_T$. In multi-label form,
  $\{c^*_{r,t} : (r,t) \in \Dom(N,\Theta)\} \in \mathcal{P}_T$. C5-S
  asserts global admissibility of the selector's output, not joint
  admissibility of local alternatives---selection in Type~S-strong
  precisely operates over alternatives that lack a common upper bound,
  so requiring local alternatives to be jointly admissible would
  conflict with the Type~S-strong hypothesis.

  \textbf{C5-E} (completion variant, \emph{compatible extension within
  the closure domain}): for every input and every pair
  $c_1, c_2 \in \mathcal{C}$ such that
  $\exists\, P_1, P_2 \in \mathcal{P}_T^{\mathrm{closure}}$ and shared
  $(r,t)$ with $c_1 \in P_1(r,t)$, $c_2 \in P_2(r,t)$: if the
  closure-domain conflict relation does not mark $\{c_1, c_2\}$ as
  forbidden at $(r,t)$, then
  $\exists\, P^+ \in \mathcal{P}_T^{\mathrm{closure}}$ with
  $\{c_1, c_2\} \subseteq P^+(r,t)$.

  The \emph{closure-domain conflict relation} is the set of forbidden
  predicate pairs whose exclusion defines
  $\mathcal{P}_T^{\mathrm{closure}}$. For $T_{ICT}$ single-timeframe,
  this is M2 (Appendix~A.1); for $T_{ICT}^{\mathrm{multi}}$, it is
  $\mathrm{M2} \cup \mathrm{M2\text{-}HTF}$ (Appendix~A.2). C5-E asserts
  that non-forbidden co-generating predicate pairs have a common
  closure-domain witness.

  \emph{Asymmetry between C5-S and C5-E.} The two variants are not
  parallel: C5-S is a forward condition on the selector's output (the
  selected assignment is admissible); C5-E is a closure-domain extension
  property (non-conflicting predicates jointly extend). The two share
  the structural intuition that canonicalization output must be globally
  consistent, but operate on different mathematical objects (a selector
  function vs.\ a set-valued closure operator).
\end{itemize}

\noindent\emph{C5 and N4 are independent conditions.} $T'$ witnesses N4
without C5-S ($\mathcal{P}_{T'} = \{\{a\},\{b\}\}$ with
$a \prec_{\mathrm{sel}} b$, so $c^* = b$ at every point, but if
selecting $b$ globally violates some other axiom in $A'$, the global
admissibility of the selected assignment fails). $T''$ witnesses C5-E
without N4 ($\mathcal{P}_{T''} = \{\{a\},\{b\},\{a,b\}\}$, with
$\{a,b\}$ jointly realizable but no theory-intrinsic order or closure
fixed point).

\subsection{Counterexample}

\begin{proposition}[N0+N1+N3 do not suffice]
\label{prop:ce}
Let $T_{ce}$ have two conclusions $c_1, c_2$: individually admissible,
axiomatically incompatible (so
$\{c_1, c_2\} \notin \mathcal{P}_{T_{ce}}$), unrelated by any ordering
axiom. Then $\mathcal{P}_{T_{ce}} = \{\{c_1\}, \{c_2\}\}$. $T_{ce}$
satisfies N0, N1, N3, but fails N4 in both forms (N4-sel: no
theory-intrinsic order; N4-comp: trivial closure has multiple fixed
points). Hence no canonicalization mechanism exists, showing
\textbf{theory-intrinsic comparability (N4) is necessary}.
\end{proposition}

\subsection{Theorem~6b (Construction Lemma) and Corollary~6b'}

\begin{theorem}[Construction lemma for canonicalization mechanisms]
\label{thm:6b}
Let $T$ satisfy AC-6'. Then:

\emph{Completion branch:}
\[
  \mathrm{N0} + \mathrm{N3} + \mathrm{N4\text{-}comp} +
  \mathrm{C5\text{-}E}
  \;\Rightarrow\;
  \Dcomp \text{ exists and satisfies (4')}
\]

\emph{Selection branch:}
\[
  \mathrm{N0} + \mathrm{N1} + \mathrm{N4\text{-}sel} +
  \mathrm{C5\text{-}S}
  \;\Rightarrow\;
  \mathrm{PhaseClassify} \text{ exists}
\]
\end{theorem}

\begin{proof}
\emph{Type~E construction.} Given the conditions, $\Dcomp$ is constructed
from M1+M2+M3 on the closure domain. N0 ensures non-emptiness; N3 ensures
termination; N4-comp provides the closure structure; C5-E ensures
co-generating predicates have common closure-domain extensions.

\emph{Type~S construction.} Given the conditions, PhaseClassify is
constructed from event dominance (C4-Q2) and MEP. N0 ensures
non-emptiness; N1 ensures finiteness of local sets; N4-sel provides the
total order
$\preceq_{\mathrm{sel}}$ defining $c^*_{r,t}$; C5-S ensures the selected
assignment $P^*$ is globally admissible.
\qed
\end{proof}

\begin{remark}[Theorem~6b is a construction lemma]
The ``forward'' arrows are constructions, not theorems with non-trivial
proof obligations beyond the constructions themselves. Conditions N1 and
N3 are tailored to the mechanism: N1 is exactly what the selector
requires for a well-founded maximum; N3 is exactly what closure requires
for termination. The deeper result is \textbf{Corollary~6b'}: under
R1+R2+R3 plus rule Soundness, (3a)-computability (interpreted as
saturated closure $F^\omega$ under elementary one-rule firing, see
below) automatically yields a mechanism satisfying \S4.1
properties (1)--(4'), so the completion operator's existence at
closure-stabilization strength is \emph{not} an additional structural
assumption but a consequence of the inference core's structure.
\end{remark}

\begin{remark}[Scope --- completion vs.\ determinization]
The completion branch establishes existence at~(4') strength. Whether
this induces full determinization (AC-6) depends on the upgrade
from~(4') to~(4), a global confluence property not supplied by N0--N4 +
C5 alone. For $T_{ICT}$, this is OQ-GC-1.
\end{remark}

\begin{corollary}[Classification of pure $I^{\mathrm{core}}$-computable
completion]
\label{cor:6bprime}
Let $F$ be the \emph{elementary} one-step $I^{\mathrm{core}}$-rule
operator on the state space $2^{\mathcal{C}\,\uparrow\,\Dom(N,\Theta)}$
---i.e., a one-step $F$-move applies exactly one rule occurrence under a
fixed fair scheduling convention, not simultaneous firing of all enabled
rules---and let $f := F^\omega$ denote its saturated closure (apply $F$
until no further rule fires). Let $I^{\mathrm{core}}$ satisfy R1+R2+R3,
and $\mathcal{P}_T$ be finite. Assume additionally:
\begin{itemize}
\item[\textbf{(Soundness)}] Every $I^{\mathrm{core}}$-rule is
  \emph{sound relative to $A$}: for every $P \in \mathcal{P}_T$ and
  every rule $\rho \in I^{\mathrm{core}}$ whose premise is satisfied by
  $P$ at some $(r,t) \in \Dom(N,\Theta)$, the state $P'$ obtained by
  adding $c_\rho$ to $P(r,t)$ is again in $\mathcal{P}_T$.
\end{itemize}
Then $f = F^\omega$ restricts to a map $\mathcal{P}_T \to \mathcal{P}_T$,
satisfies properties~(1)--(4') of \S4.1, and is therefore Class~C.
\end{corollary}

\begin{proof}
\emph{Saturation well-definedness.} We take a one-step $F$-move to mean
one elementary $I^{\mathrm{core}}$-rule firing, not simultaneous firing
of all enabled rules. By R1, every nontrivial $F$-move adds at least one
conclusion and removes none. Since
$|\mathcal{C}| \cdot |\Dom(N,\Theta)| < \infty$, no strictly increasing
$F$-trajectory has length exceeding this bound. Hence $F^\omega(P)$
exists as a finite-step saturation for every state $P$, under the chosen
scheduling convention.

\emph{Order independence (under R1+R2+R3 + finite + fair scheduling).}
In the pure (3a) positive-rule setting, under fair elementary
scheduling, saturation is order-independent: every conclusion whose
positive premises eventually become true is eventually added (R1 + fair
scheduling), no rule is conditioned on the \emph{absence} of a derived
predicate (R2, no negation-as-failure), and no rule retracts conclusions
(R3). Hence $F^\omega(P)$ is independent of the fair scheduling
convention.

\emph{Restriction to $\mathcal{P}_T$.} By Soundness, each elementary
$F$-step from a $\mathcal{P}_T$-state yields a $\mathcal{P}_T$-state.
Finite composition along the saturation trajectory gives
$F^\omega(P) \in \mathcal{P}_T$ for every $P \in \mathcal{P}_T$. Hence
$f$ restricts to $\mathcal{P}_T \to \mathcal{P}_T$.

\emph{Monotonicity and extensivity.} By Lemma~C (\S8.2), R1+R2+R3 imply
$f$ is monotone and extensive on the finite poset
$(\mathcal{P}_T, \Pspec)$. (Lemma~C does not require Soundness---it is a
syntactic consequence of R1+R2+R3.)

\emph{Per-seed convergence~(4').} By the Monotonic Exhaustion Lemma
(\S8.2), an extensive operator on a finite poset reaches a fixed point
in at most $|\mathcal{P}_T|$ steps. Combined with the saturation
interpretation, $f(P) = F^\omega(P)$ is itself a fixed point:
$f(f(P)) = F^\omega(F^\omega(P)) = F^\omega(P) = f(P)$.

\emph{Idempotence.} From the previous line, $f \circ f = f$ as a
saturation tautology.

All four Class~C properties (1)--(4') hold.
\qed
\end{proof}

\begin{remark}[Soundness is independent of R1+R2+R3]
\label{rem:soundness-independent}
R1+R2+R3 are \emph{syntactic} constraints on rule shape: rules add
positive conclusions (R1), do not negate-as-failure (R2), and do not
retract (R3). Soundness is the \emph{semantic} condition that the
conclusions added by rules remain consistent with the axioms~$A$. The
two are independent: a positive non-retractive rule ``$a \Rightarrow b$''
satisfies R1+R2+R3 syntactically, but is unsound relative to an axiom
forbidding $\{a, b\}$, and saturated closure of the seed $\{a\}$ would
then exit the admissible family. Soundness must therefore be assumed
separately and cannot be derived from R1+R2+R3 alone. For $T_{ICT}$, the
soundness condition is supplied by M1: closure rules are checked to add
only conclusions consistent with the M2 conflict relation that defines
$\mathcal{P}_{ICT}^{\mathrm{closure}}$.

The Soundness condition above is \emph{single-rule} soundness (each
elementary rule firing preserves $\mathcal{P}_T$). Combined with the
elementary-step interpretation of $F$, this is sufficient for the
saturated closure $F^\omega$ to preserve $\mathcal{P}_T$ by finite
composition. A reader interpreting $F$ as a simultaneous-firing operator
would need a stronger condition: rules $a \Rightarrow b$ and
$a \Rightarrow c$ are individually sound under an axiom forbidding
$\{b,c\}$, but the simultaneous step from $\{a\}$ to $\{a,b,c\}$ is
unsound. The present paper uses elementary-step $F$ throughout.
\end{remark}

\begin{remark}[One-step vs.\ saturated form]
The (3a) realization is the saturated closure $f = F^\omega$ under
elementary one-rule firing, not the one-step operator $F$. The one-step
operator is generally not idempotent ($F(F(P)) \neq F(P)$ when more than
one rule application is needed). Saturation collapses iteration into a
single semantic step. All Class~C results below assume the saturated
form for (3a) operators.
\end{remark}

\begin{remark}
Corollary~6b' establishes Class~C at closure-stabilization
strength~(4'). It does not establish property~(4) or AC-6. Mechanisms
involving guards or scheduling (condition (3c)) require direct \S4.1
verification, as in \S6.2 for $\DcompHTF$.
\end{remark}

\subsection{Summary}

\noindent\fbox{\begin{minipage}{0.97\linewidth}\small
Canonicalization $=$ Closure $+$ Comparability $+$ Joint Admissibility

\medskip
\noindent Theorem~6b (construction lemma):

Type~E: $\mathrm{N0} + \mathrm{N3} + \mathrm{N4\text{-}comp} +
\mathrm{C5\text{-}E} \Rightarrow \Dcomp$ exists at~(4').

Type~S-strong: $\mathrm{N0} + \mathrm{N1} + \mathrm{N4\text{-}sel} +
\mathrm{C5\text{-}S} \Rightarrow$ PhaseClassify exists at AC-6.

\medskip
\noindent Corollary~6b' (substantive content):
$\mathrm{R1}{+}\mathrm{R2}{+}\mathrm{R3} +$ Soundness $+$
(3a)-computability (saturated closure $F^\omega$ under elementary
one-rule firing) $\Rightarrow$ Class~C at~(4') via Lemma~C and
Monotonic Exhaustion.

\medskip
N0 universal; N1/N3 mechanism-specific. C5-E and C5-S are
mechanism-specific variants of joint admissibility, not parallel
restatements: C5-S is global admissibility of the selector's output;
C5-E is compatible extension within the closure domain. Completion
branch yields~(4'); upgrade to AC-6 requires (4) (OQ-GC-1 for ICT).

\medskip
Type~E $\to$ completion (at~(4')); Type~S-strong $\to$ selection (at
AC-6). Generic Type~S $\to$ selection: not proven (OQ-TypeS-Imp).

\medskip
\emph{Not all canonicalization is operator-theoretic.}
\end{minipage}}

\section{Examples}

\subsection{Wyckoff (Type~S-strong, Complete)}

Event dominance (C4-Q2) gives
$\mathrm{Phase\_A} \prec_{\mathrm{sel}} \cdots \prec_{\mathrm{sel}}
\mathrm{Phase\_E}$. MEP ensures the selected assignment is globally
admissible. $T_{Wy}^{\mathrm{canon}} = T_{Wy}^{v2} +
\mathrm{PhaseClassify} \in \mathrm{AC\text{-}6}$ (full domain).

\subsection{ICT (Type~E) --- Single-Timeframe and Multi-HTF}
\label{sec:ict}

\textbf{Single-timeframe.} M1+M2 define $\PICT^{\mathrm{closure}}$; M3
ensures $\Dcomp$ terminates. Forbidden pairs (M2):
\begin{align*}
  &\{\mathrm{LTF\_BOS_{up}},\mathrm{LTF\_BOS_{down}}\},\quad
   \{\mathrm{LTF\_bullish\_cont},\mathrm{LTF\_bearish\_cont}\},\\
  &\{\mathrm{LTF\_BOS_{up}},\mathrm{LTF\_bearish\_cont}\},\quad
   \{\mathrm{LTF\_BOS_{down}},\mathrm{LTF\_bullish\_cont}\}.
\end{align*}
$T_{ICT}^{\mathrm{canon}}$ achieves~(4') on $\PICT^{\mathrm{closure}}$.
Whether it achieves~(4) and hence AC-6 is conditional on OQ-GC-1.

\textbf{Multi-timeframe extension.}
$\DcompHTF := D_L \circ D_H: \PICT \to \PICT^{\mathrm{HTF}} \to
\PICT^{\mathrm{multi}}$.

LTF inference rules (positive premises only):
\begin{center}
\begin{tabular}{lll}
\toprule
Rule & Premise & Conclusion \\
\midrule
$\rho_1$ & $\mathrm{close} > \mathrm{swing\_high}_{\tau_{LTF}}$
  & $\mathrm{LTF\_BOS_{up}}$ \\
$\rho_2$ & $\mathrm{close} < \mathrm{swing\_low}_{\tau_{LTF}}$
  & $\mathrm{LTF\_BOS_{down}}$ \\
$\rho_3$ & $\mathrm{close} > \mathrm{prev\_close}$
  & $\mathrm{LTF\_bullish\_cont}$ \\
$\rho_4$ & $\mathrm{close} < \mathrm{prev\_close}$
  & $\mathrm{LTF\_bearish\_cont}$ \\
\bottomrule
\end{tabular}
\end{center}

\noindent\textit{Premise mutual exclusion.} Within each priority tier,
rule premises are pairwise mutually exclusive on observed data:
\begin{itemize}
\item BOS tier: $\mathrm{close}$ cannot simultaneously satisfy
  $> \mathrm{swing\_high}$ and $< \mathrm{swing\_low}$ (assuming
  $\mathrm{swing\_low} \leq \mathrm{swing\_high}$).
\item Continuation tier: $\mathrm{close}$ cannot simultaneously satisfy
  $> \mathrm{prev\_close}$ and $< \mathrm{prev\_close}$. The boundary
  case $\mathrm{close} = \mathrm{prev\_close}$ triggers neither rule
  (non-event by convention).
\end{itemize}

Cross-timeframe forbidden pairs (M2-HTF):
\begin{align*}
  &\{\mathrm{HTF\_BOS_{down}},\mathrm{LTF\_BOS_{up}}\},\quad
   \{\mathrm{HTF\_BOS_{up}},\mathrm{LTF\_BOS_{down}}\},\\
  &\{\mathrm{HTF\_bearish\_bias},\mathrm{LTF\_bullish\_cont}\},\quad
   \{\mathrm{HTF\_bullish\_bias},\mathrm{LTF\_bearish\_cont}\}.
\end{align*}

\begin{definition}[Guard mechanism]
For state $S$, an LTF rule $\rho$ with conclusion $c_\rho$ is
\emph{guard-suppressed} iff: (1) $\exists\, c' \in \pi_L(S)$ with
$\{c', c_\rho\} \in \mathrm{M2}$, or (2) $\exists\, h \in \pi_H(S)$ with
$\{h, c_\rho\} \in \mathrm{M2\text{-}HTF}$.
\end{definition}

\begin{definition}[Rule priority for $D_L$]
$\rho_1, \rho_2$ (BOS) have higher priority than $\rho_3, \rho_4$
(continuation).
\end{definition}

\begin{definition}[Reduction relation for $D_L$]
On admissible $S \in \PICT^{\mathrm{HTF}}$, $S \xrightarrow{\rho} S'$
iff:
(a) $\rho \in A_{ICT}^{LTF}$;
(b) premise$(\rho)$ satisfied in $S$;
(c) $\rho$ not guard-suppressed on $S$;
(d) $\rho$ has highest priority among rules satisfying (a)--(c);
(e) $S' = S \cup \{c_\rho\}$.
\end{definition}

\begin{lemma}[Functionality of $\to$]
\label{lem:functionality}
For every admissible $S \in \PICT^{\mathrm{HTF}}$, there is at most one
$S'$ with $S \xrightarrow{\rho} S'$.
\end{lemma}

\begin{proof}
Two-part argument:
\emph{(1) Across tiers:} (d) selects only the highest-priority tier
among applicable rules. If any BOS rule satisfies (a)--(c), no
continuation rule passes (d).
\emph{(2) Within a tier:} Premise mutual exclusion gives at most one
rule per tier with premise (b) satisfied for given $(N, \Theta)$.
Combined, at most one rule satisfies (a)--(d).
\qed
\end{proof}

\begin{remark}
Premise mutual exclusion provides within-tier uniqueness. Guards
(condition (c)) play a different role: admissibility preservation
(Lemma~\ref{lem:6b}), not within-tier uniqueness.
\end{remark}

\begin{lemma}[Single-step admissibility preservation under guard]
\label{lem:6b}
If $S$ is admissible and $S \xrightarrow{\rho} S'$, then $S'$ is
admissible.
\end{lemma}

\begin{proof}
\emph{Same-timeframe.} If $\{c', c_\rho\} \in \mathrm{M2}$ for some
$c' \in \pi_L(S)$, then $\rho$ is
guard-suppressed---contradicting~(c).
\emph{Cross-timeframe.} $c_\rho \in \mathcal{C}_{LTF}$ (since
$\rho \in A_{ICT}^{LTF}$), so $\pi_H(S') = \pi_H(S)$. If
$\{c_H, c_\rho\} \in \mathrm{M2\text{-}HTF}$ with $c_H \in \pi_H(S)$,
then $\rho$ is guard-suppressed---contradicting~(c). Otherwise
$c_H \notin \pi_H(S')$, so the forbidden pair is not in $S'$.
\qed
\end{proof}

\noindent\textbf{Verification of \S4.1 properties for $\DcompHTF$.}
$\DcompHTF$ is the saturated closure $(D_L \circ D_H)^\omega$ of its
underlying one-step guarded operator; the verification below is for this
saturated form. We verify (1)--(4') by direct argument from the
reduction relation, without appeal to external lemmas.

\begin{enumerate}
\item \emph{Extensivity.}
$P \Pspec D_H(P) \Pspec D_L(D_H(P))$ by additive updates ($D_H$ adds
HTF predicates; $D_L$ adds LTF predicates via $S' = S \cup \{c_\rho\}$).

\item \emph{Admissibility preservation.}
Induction on reduction length using Lemma~\ref{lem:6b} for each $D_L$
step. $D_H$ admissibility is direct (no guard layer; $D_H$ does not
modify LTF predicates).

\item \emph{Idempotence.} Fix $P$, let
$P^* := \DcompHTF(P) = D_L(D_H(P))$. We show $\DcompHTF(P^*) = P^*$.

\emph{Step~1 ($D_H(P^*) = P^*$).} $D_H$'s rules reference external HTF
evidence, not derived predicates. Since $\pi_H(P^*) = \pi_H(D_H(P))$
($D_L$ does not add HTF predicates), every $D_H$-applicable rule has
already fired in $P^*$.

\emph{Step~2 ($D_L(P^*) = P^*$).} By construction, $D_L$ runs to fixed
point, so no $\rho$ satisfies (a)--(d) on $P^*$.

\emph{Combining.} $\DcompHTF(P^*) = D_L(D_H(P^*)) = D_L(P^*) = P^*$.

\item \emph{Per-seed convergence~(4').}
$D_H$ terminates: HTF rules are finitely many on $\mathcal{C}_{HTF}$.
$D_L$ on $D_H(P)$ terminates: each step strictly increases
$|S \cap \mathcal{C}_{LTF}|$ (extensive addition of new predicates,
bound $|\mathcal{C}_{LTF}| \cdot |\Dom(N,\Theta)|$); the relation is
functional (Lemma~\ref{lem:functionality}). Each seed converges to a
fixed point of $\DcompHTF$ in $\PICT^{\mathrm{multi}}$ within
$|\mathcal{C}_{HTF}| + |\mathcal{C}_{LTF}| \cdot |\Dom(N,\Theta)|$
iterations. The fixed point is seed-dependent; whether all seeds
converge to the same fixed point is OQ-GC-1.
\end{enumerate}

\begin{remark}[Monotonicity of $\DcompHTF$ in $\Pspec$ --- fails]
\label{rem:monotonicity-fails}
$\DcompHTF$ is \textbf{not monotone} under the naive specification order
$\Pspec$. \emph{Concrete witness:} $P = \emptyset$,
$Q = \{\mathrm{HTF\_bearish\_bias}\}$, at an input where
$\mathrm{close} > \mathrm{prev\_close}$ (so $\rho_3$'s premise is
satisfied). Then $P \Pspec Q$ trivially, and:
\begin{itemize}
\item $\DcompHTF(P)$: $\rho_3$ is not guard-suppressed (no HTF predicate
  present), so $\rho_3$ fires, yielding
  $\DcompHTF(P) \supseteq \{\mathrm{LTF\_bullish\_cont}\}$.
\item $\DcompHTF(Q)$: $\rho_3$ is M2-HTF-suppressed
  ($\{\mathrm{HTF\_bearish\_bias}, \mathrm{LTF\_bullish\_cont}\} \in
  \mathrm{M2\text{-}HTF}$), so $\rho_3$ does not fire and
  $\DcompHTF(Q) = Q = \{\mathrm{HTF\_bearish\_bias}\}$.
\end{itemize}
Hence $\mathrm{LTF\_bullish\_cont} \in \DcompHTF(P)$ but
$\mathrm{LTF\_bullish\_cont} \notin \DcompHTF(Q)$, so
$\DcompHTF(P) \nsubseteq_{\mathrm{spec}} \DcompHTF(Q)$. This failure is
intrinsic to (3c)-mechanisms with conflict-suppression guards: adding a
predicate to $Q$ may activate a guard inactive on $P$, blocking a
derivation that fires on $P$. Monotonicity is therefore not a generic
property of (3c)-realized canonicalization mechanisms.

This non-monotonicity does not threaten the \S4.1 verification above:
extensivity, idempotence, admissibility preservation, and~(4') are all
established without monotonicity. The Monotonic Exhaustion Lemma
(\S8.2) likewise establishes~(4') from finite + extensive alone.
\end{remark}

\begin{definition}[$D_H$ architectural scope]
$D_H$ applies HTF inference rules whose premises reference external HTF
evidence and do not depend on derived LTF predicates. $D_H$ does not
employ an M2-HTF guard---cross-timeframe conflict resolution is the
responsibility of $D_L$ via the Guard Mechanism. Staged ordering (\S7)
resolves the resulting non-commutativity.
\end{definition}

$T_{ICT}^{\mathrm{multi}} = T_{ICT}^{v2} + \DcompHTF$ achieves closure
stabilization~(4') on $\PICT^{\mathrm{multi}}$. Full determinization
(AC-6, requiring~(4)) is conditional on OQ-GC-1.

\section{Non-Commutativity of Staged Completion}

\begin{proposition}[Structural non-commutativity]
\label{prop:noncomm}
There exists $P \in \PICT$ such that
$(D_L \circ D_H)(P) \in \PICT^{\mathrm{multi}}$ while
$(D_H \circ D_L)(P) \notin \PICT^{\mathrm{multi}}$.
\end{proposition}

\begin{proof}
Let $P = \varnothing$; assume $\mathrm{close} > \mathrm{prev\_close}$,
no BOS condition holds, HTF bearish BOS evidence confirmed.

\textbf{Path~A} ($D_L \circ D_H$, canonical):
(1) $D_H(P) = \{\mathrm{HTF\_bearish\_bias}\}$;
(2) $D_L$: BOS rules do not fire; $\rho_3$ is guard-suppressed
(since $\{\mathrm{HTF\_bearish\_bias}, \mathrm{LTF\_bullish\_cont}\}
\in \mathrm{M2\text{-}HTF}$).
$R_A = \{\mathrm{HTF\_bearish\_bias}\} \in \PICT^{\mathrm{multi}}$.

\textbf{Path~B} ($D_H \circ D_L$, non-canonical):
(1) $D_L(\varnothing) = \{\mathrm{LTF\_bullish\_cont}\}$ ($\rho_3$
fires, no active guard);
(2) $D_H$ fires (premise external; no M2-HTF guard).
$R_B = \{\mathrm{HTF\_bearish\_bias}, \mathrm{LTF\_bullish\_cont}\}$
violates M2-HTF, so $R_B \notin \PICT^{\mathrm{multi}}$.
\qed
\end{proof}

\begin{corollary}[Canonical ordering uniqueness within the staged
architecture]
\label{cor:noncomm-uniqueness}
Within the two-operator raw staged architecture
$\{D_L \circ D_H, D_H \circ D_L\}$, with no repair operator, no
M2-HTF guard added to $D_H$, and no joint fixed-point construction,
HTF-first $D_L \circ D_H$ is the only admissibility-preserving
linearization.
\end{corollary}

\begin{remark}[Scope of uniqueness]
The corollary scopes uniqueness to the specified two-operator linear
architecture and the two raw compositions. Other admissibility-preserving
multi-stage architectures are not excluded: alternative orderings
combined with repair operators, joint fixed-point constructions, or
$D_H$ guarded against M2-HTF could in principle preserve admissibility
under different orderings. The witness establishes structural
non-commutativity within the declared architecture; a global uniqueness
claim across all staged architectures is not asserted.
\end{remark}

\section{Mechanism Classification and Categorical Structure}

\emph{This section classifies, conditional on a realization-coverage
assumption (OQ-Realization) and per-realization branch hypotheses, every
theory-intrinsic canonicalization mechanism realized within the
(3a)/(3b)/(3c) grammar into Class~C (closure) or Class~S (selection).
The classification is at closure-stabilization strength~(4'); whether
Class~C operators additionally achieve property~(4) or AC-6 is
instance-specific.}

\subsection{Definitions}

\begin{definition}[Theory-intrinsic operation]
\label{def:theory-intrinsic}
$f: \mathcal{P}_T \to \mathcal{P}_T$ is \emph{theory-intrinsic} if:
\begin{itemize}
\item[(I1)] $f$ is definable from $T = (\Sigma, A, I)$---i.e., $f$ is
  given by a fixed expression in the signature, axioms, and inference
  policy of $T$, with no external parameters beyond input
  $(N,\Theta,\Pi)$;
\item[(I2)] $f$ is invariant under automorphisms of $T$: if $\sigma$ is
  an automorphism of $(\Sigma, A)$, then
  $f(\sigma \cdot P) = \sigma \cdot f(P)$ for every $P \in \mathcal{P}_T$.
\end{itemize}
This definition is independent of any mechanism taxonomy. Conditions
(I1)--(I2) specify \emph{what} makes an operation theory-intrinsic; the
trichotomy below specifies \emph{how} such operations are realized,
used as the case partition for Theorem~U3.
\end{definition}

\noindent\textit{Examples.} $\Dcomp$ (ICT closure) and PhaseClassify
(Wyckoff event dominance) are both theory-intrinsic---the former by
(3a) realization (saturated closure under elementary firing), the
latter by (3b). $\DcompHTF$ is theory-intrinsic by (3c) realization
(saturated guarded $I^{\mathrm{core}}$).

\medskip\noindent\textbf{Realizations} (case partition for Theorem~U3).

\begin{itemize}
\item \textbf{(3a)} $f$ is \emph{$I^{\mathrm{core}}$-computable} as the
  saturated closure $f = F^\omega$ under elementary one-rule firing:
  $f(P)$ is the result of applying $I^{\mathrm{core}}$ rules to $P$ one
  rule at a time under a fair scheduling convention until no further
  rule fires, with no guard or scheduling discipline beyond fairness.
  The one-step operator $F$ is not itself the (3a) realization;
  saturation is built in.
\item \textbf{(3b)} $f$ is \emph{selector-computable} from axiom-induced
  comparability: $f(P)(r,t) := \max_{\preceq_{\mathrm{sel}}} A_{r,t}$
  where $\preceq_{\mathrm{sel}}$ is theory-intrinsic.
\item \textbf{(3c)} $f$ is \emph{guarded $I^{\mathrm{core}}$ (saturated
  form)}: there is a one-step guarded operator $G$ such that $f := G^\omega$,
  where $G$ applies $I^{\mathrm{core}}$ rules under a guard mechanism
  and/or scheduling discipline that:
  (i) suppresses or orders $I^{\mathrm{core}}$ rule applications based
  on state and axiom-derived conflict/priority relations;
  (ii) adds no conclusions beyond those derivable by $I^{\mathrm{core}}$
  rules---every predicate in $G(P) \setminus P$ is the conclusion of
  some $\rho \in I^{\mathrm{core}}$;
  and $f(P) = G^\omega(P)$ is obtained by repeated application of $G$
  until no guarded step fires.
\end{itemize}

\noindent (3a) $\subset$ (3c) (pure $I^{\mathrm{core}}$ saturation is
the special case $G = F$ with trivial guard).

\medskip\noindent\textbf{Realization-coverage assumption (OQ-Realization).}
Theorem~U3 below is conditional on the assumption that every valid
theory-intrinsic canonicalization mechanism is realized by exactly one
of (3a), (3b), or (3c). This is an open question, tracked as
OQ-Realization in \S8.5; whether the trichotomy is provably exhaustive
for all theory-intrinsic operations definable in MST is not established
in this paper. Theorem~U3's conclusion is therefore explicitly scoped
to mechanisms realized within this grammar.

\begin{definition}[Class~C --- Completion]
$f$ is \emph{Class~C} if $f$ satisfies properties~(1)--(3) and~(4') of
\S4.1. Class~C is a mechanism class; it does not imply~(4) or AC-6.
\end{definition}

\begin{definition}[Class~S --- Selection]
$f$ is \emph{Class~S} if $f$ is induced by a canonical selector---i.e.,
$f$ is the (3b) realization with respect to some
$\preceq_{\mathrm{sel}}$. Class~S directly achieves AC-6 and need not
satisfy extensivity in $\Pspec$.
\end{definition}

\subsection{Asymmetric Reduction (Lemma R2*)}

\begin{lemma}[Asymmetric reduction]
\label{lem:R2star}
Let $\mathcal{P}_T$ be a dcpo (which holds under M3) and let
$f: \mathcal{P}_T \to \mathcal{P}_T$ be Scott-continuous,
theory-intrinsic, and uniqueness-inducing in the sense of \S4.1.

\emph{(i) No outward drift below $P^*$.} If $\exists\, P \prec P^*$
with $P \prec f(P)$, then $\{f^n(P)\}_n$ is a directed increasing
chain; Scott-continuity gives $f(P_\infty) = P_\infty$, and uniqueness
forces $P_\infty = P^*$. Hence uniqueness-inducing upward motion is
closure/completion-like toward $P^*$.

\emph{(ii) Asymmetry above $P^*$.} For $P \succ P^*$, descending-limit
arguments do not apply in a general dcpo (infima of decreasing chains
are not guaranteed). Uniqueness from above cannot be derived via
Scott-continuous downward convergence alone.
\end{lemma}

\begin{remark}[R2* role]
R2* is a structural observation, not a new theorem. It characterizes
the upward/downward asymmetry of uniqueness-inducing Scott-continuous
operators. The mechanism reduction conclusion---every
uniqueness-inducing theory-intrinsic operator is closure-like or
selection-like---is the content of Theorem~U3, not of R2*.
\end{remark}

\subsection{Conditional Classification within the MST Realization Grammar
(Theorem~U3)}

\begin{theorem}[Conditional primitive mechanism classification under
R1+R2+R3]
\label{thm:U3}
Let $T$ have $I^{\mathrm{core}}$ satisfying R1+R2+R3 and finite
$\mathcal{C}$. Let $f: \mathcal{P}_T \to \mathcal{P}_T$ be
theory-intrinsic (Definition~\ref{def:theory-intrinsic}, conditions
(I1)--(I2)).

\emph{Assume the realization-coverage hypothesis (OQ-Realization): $f$
is realized by one of the \S8.0 forms (3a), (3b), or (3c).}

\emph{Assume the relevant per-realization branch hypotheses:}
\begin{itemize}
\item \textbf{(3a)-branch}: R1+R2+R3 + Soundness (\S5.3
  Corollary~\ref{cor:6bprime}) + finite $\mathcal{P}_T$;
  $F$-step is elementary one-rule firing under fair scheduling.
\item \textbf{(3b)-branch}: total theory-intrinsic selector
  $\preceq_{\mathrm{sel}}$ + global compatibility of selected output
  (C5-S, \S5.1).
\item \textbf{(3c)-branch}: R1+R2+R3 + (G2) functionality + (G3)
  guard-stability + Soundness (\S8.2 Lemma~\ref{lem:G}).
\end{itemize}

Then:
\begin{itemize}
\item (3a) and (3c) realizations yield $f \in$ Class~C;
\item (3b) realization yields $f \in$ Class~S.
\end{itemize}
Within the (3a)/(3b)/(3c) realization grammar, every theory-intrinsic
$f$ satisfying the corresponding per-branch hypotheses is Class~C or
Class~S.
\end{theorem}

\begin{remark}[Scope of U3]
U3 is conditional. Without the realization-coverage assumption, U3 does
not establish that no third primitive mechanism exists in MST; it
establishes that, within the declared grammar and under the per-branch
hypotheses, no third class arises. The per-branch hypotheses
(R1+R2+R3 + Soundness + elementary $F$ for (3a); $\preceq_{\mathrm{sel}}$
+ C5-S for (3b); R1+R2+R3 + (G2) + (G3) + Soundness for (3c)) are
explicitly stated in the theorem hypothesis; U3 is not a tautology in
$\mathrm{Class\ C} \cup \mathrm{Class\ S}$, but a derivation of
Class~C / Class~S membership from per-realization assumptions.

Mechanisms outside the grammar (e.g., hybrids mixing selector outputs
with $I^{\mathrm{core}}$ rule conclusions in ways not reducible to
(3a)/(3b)/(3c), or operators encoding cardinality-minimization, orbit
selection under automorphism groups, or fixed-point selection via
secondary invariants) are not classified by U3. Whether such mechanisms
are theory-intrinsic and, if so, whether they fall outside Class~C and
Class~S, is part of OQ-Realization.

U3 also does \textbf{not} assert that any classified mechanism is
uniqueness-inducing (property (4)) or achieves determinization (AC-6);
whether a specific Class~C operator does so is instance-dependent and
governed by global confluence, not by mechanism class.
\end{remark}

\begin{proof}
\emph{Case partition.} By (3a) $\subset$ (3c), pure
$I^{\mathrm{core}}$-computable $f$ satisfies both (3a) and (3c); we
treat such $f$ under Case~1 to exploit the stronger conclusions via
Lemma~C and Monotonic Exhaustion. Case~3 handles (3c)-but-not-(3a)
mechanisms.

\emph{Case~1 ((3a)-branch).} Lemma~\ref{lem:C} gives monotonicity and
extensivity of the saturated closure under R1+R2+R3. Single-rule
Soundness combined with elementary-step $F$ gives admissibility
preservation along the saturation trajectory (finite composition of
admissibility-preserving steps). By Lemma~\ref{lem:monexh} (Monotonic
Exhaustion), finite extensivity gives~(4'). Saturation gives
idempotence ($F^\omega \circ F^\omega = F^\omega$). Hence $f$ is
Class~C.

\emph{Case~2 ((3b)-branch).} By the Class~S definition (\S8.0), an $f$
realized as (3b) with theory-intrinsic $\preceq_{\mathrm{sel}}$ is
Class~S; C5-S gives global admissibility of the selector's output.

\emph{Case~3 ((3c)-but-not-(3a)).} By Lemma~\ref{lem:G} (Guard
Preservation) applied to $G^\omega$ under (G2)+(G3)+Soundness, $f$ is
extensive, admissibility-preserving, idempotent (by saturation), and
satisfies~(4'). Hence $f$ is Class~C by the property-based definition.
The realization is Class~C, not Class~S, since by (3c)(ii) every added
predicate is an $I^{\mathrm{core}}$ conclusion, not a selector output.
\qed
\end{proof}

\begin{lemma}[Operator monotonicity from R1+R2+R3, (3a)-scope]
\label{lem:C}
Let $f$ be the saturated closure $F^\omega$ of an
$I^{\mathrm{core}}$-computable operator in the sense of (3a) (elementary
$F$-step, fair scheduling). Under R1+R2+R3, $f$ is monotone and extensive
on finite $\mathcal{P}_T$.
\end{lemma}

\begin{proof}
Suppose $P \Pspec Q$. To show $f(P) \Pspec f(Q)$, we show every
conclusion in $f(P)$ is in $f(Q)$.

A rule $\rho \in I^{\mathrm{core}}$ has a premise that, by R2, is a
positive conjunction of two classes:

(Class~A) Conditions on observed data: comparisons or pattern matches on
$(N, \Theta)$ alone, e.g., $\mathrm{close} > \mathrm{swing\_high}$.
These depend only on $(N, \Theta)$, not on $P$ or $Q$. Equivalent across
all $P, Q \in \mathcal{P}_T$ with the same $(N, \Theta)$.

(Class~B) Conditions on derived predicates: positive presence
conditions ``$c \in P(r,t)$''. By R2, no Class~B condition is of the
form ``$c \notin P(r,t)$''.

For Class~B: if ``$c \in P(r,t)$'' holds, then $P \Pspec Q$ gives
$P(r,t) \subseteq Q(r,t)$, so ``$c \in Q(r,t)$'' also holds. Hence
positive presence conditions monotone-preserve under $\Pspec$.

Combining: every premise triggering $\rho$ on $P$ also triggers $\rho$
on $Q$. By R1, $\rho$ adds $c_\rho$. By R3, no rule retracts. Hence
every conclusion in $F^\omega(P) = f(P)$ is in $f(Q)$.

Extensivity: each $I^{\mathrm{core}}$ rule application adds zero or
more conclusions, never removing any, so $P \Pspec f(P)$.
\qed
\end{proof}

\begin{remark}[Lemma~C does not extend to (3c) under naive $\Pspec$]
\label{rem:lemmaC-nonextension}
Lemma~C uses pure $I^{\mathrm{core}}$ computation: adding facts to $P$
can only add rule applications, never remove them. Once a guard layer
is introduced (3c), adding facts can \emph{activate} a
previously-inactive suppression---which removes a firing that would
have happened in the smaller state. Hence monotonicity in $\Pspec$ is
not a generic property of (3c)-realized mechanisms (cf.\
Remark~\ref{rem:monotonicity-fails} in \S6.2 for an explicit ICT
counterexample). This is not a defect: Lemma~G below establishes (4')
without requiring monotonicity.
\end{remark}

\begin{lemma}[Monotonic exhaustion (per-seed)]
\label{lem:monexh}
Let $f: \mathcal{P}_T \to \mathcal{P}_T$ be extensive, with
$\mathcal{P}_T$ finite. Then for every seed
$P_0 \in \mathcal{P}_T$ there exist
$k(P_0) \leq |\mathcal{P}_T|$ and $P^*(P_0) \in \mathcal{P}_T$ with:
\[
  P_0 \Pspec f(P_0) \Pspec \cdots \Pspec f^{k(P_0)}(P_0) = P^*(P_0)
\]
and $f(P^*(P_0)) = P^*(P_0)$. That is, $f$ satisfies property~(4').
\end{lemma}

\begin{proof}
Extensivity gives the ascending chain in the finite poset
$(\mathcal{P}_T, \Pspec)$. Any ascending chain in a finite poset
stabilizes within $|\mathcal{P}_T|$ steps, yielding a fixed point.
Monotonicity is not needed for this conclusion; it is recorded as a
companion structural property where it holds (Lemma~\ref{lem:C},
(3a)-scope).
\qed
\end{proof}

\begin{remark}[Scope]
The lemma establishes property~(4') only---different seeds may converge
to different fixed points. Promoting~(4') to~(4) requires an additional
global confluence hypothesis on $f$'s ARS, which is theory-specific and
not implied by extensivity and finiteness alone.
\end{remark}

\begin{example}[Witnessing (4') $\not\Rightarrow$ (4)]
$\mathcal{P}_T = \{a, b\}$ incomparable under $\Pspec$, $f =
\mathrm{id}_{\mathcal{P}_T}$. Then $f$ is extensive, idempotent, with
two distinct fixed points. Hence extensive + finite does not imply
uniqueness.
\end{example}

\begin{lemma}[Guard preservation under R1+R2+R3, (3c)-scope]
\label{lem:G}
Let $G$ be a one-step guarded operator satisfying \S8.0 conditions
(3c)(i)--(ii), and let $f := G^\omega$ denote its saturated closure.
Suppose:
\begin{itemize}
\item[(G2)] \emph{Functionality:} at each state, at most one rule
  (highest-priority unsuppressed rule with satisfied premise) fires
  next under $G$.
\item[(G3)] \emph{Additive-stable guard:} along any $G$-trajectory,
  predicates added by $G$ at state $S$ do not subsequently
  un-suppress rules that were $g$-suppressed at earlier states (the
  conflict relation defining $g$ is monotone under $\subseteq$ at the
  witness level).
\item[\textbf{(Soundness)}] for every $P \in \mathcal{P}_T$ and every
  $G$-step from $P$ yielding $P'$, $P' \in \mathcal{P}_T$.
\end{itemize}
Then $f = G^\omega$ is well-defined on $\mathcal{P}_T$, restricts to a
map $\mathcal{P}_T \to \mathcal{P}_T$, and satisfies extensivity,
admissibility preservation, idempotence, and per-seed convergence~(4')
on finite $\mathcal{P}_T$.
\end{lemma}

\begin{proof}
\emph{Step~1 --- $G$-trajectory terminates.} Let $P_0 = P$ and
$P_{n+1} = G(P_n)$ whenever a $G$-step is available at $P_n$. By R1
and (3c)(ii), each non-trivial $G$-step adds at least one
$I^{\mathrm{core}}$ conclusion and removes none, so
$|P_n \cap \mathcal{C} \cdot \Dom(N,\Theta)|$ strictly increases along
the trajectory. Since
$|\mathcal{C}| \cdot |\Dom(N,\Theta)| < \infty$, the trajectory has
length at most $|\mathcal{C}| \cdot |\Dom(N,\Theta)|$ and reaches a
state $P^* = P_N$ at which no $G$-step is available. By (G2), the
trajectory is functional and hence unique.

\emph{Step~2 --- $f := G^\omega$ is well-defined.} By Step~1 the
$G$-trajectory terminates, and by (G2) it is functional; hence
$G^\omega(P)$ is well-defined as the terminal state of the unique
$G$-trajectory from $P$. By (G3), guard suppression is stable along
the trajectory: predicates added later do not un-suppress rules that
were suppressed earlier by positive conflict witnesses. ((G3) is not
needed for well-definedness itself---that follows from finite additive
termination plus (G2)---but it records the additional stability
property satisfied by the ICT guard.)

\emph{Step~3 --- Restriction to $\mathcal{P}_T$ and admissibility
preservation.} By Soundness, each $G$-step from a $\mathcal{P}_T$-state
yields a $\mathcal{P}_T$-state. Finite composition along the trajectory
gives $f(P) \in \mathcal{P}_T$ for every $P \in \mathcal{P}_T$.

\emph{Step~4 --- Extensivity.}
$P_0 \Pspec P_1 \Pspec \cdots \Pspec P_N = f(P)$ by (3c)(ii) and R1, so
$P \Pspec f(P)$.

\emph{Step~5 --- Idempotence.} By construction, no $G$-step is available
at $f(P) = P^*$, so applying $G^\omega$ again yields $f(P)$ itself.
Hence $G^\omega(G^\omega(P)) = G^\omega(P)$, i.e.\ $f \circ f = f$.
(Idempotence is a saturation tautology and does not require G3.)

\emph{Step~6 --- Per-seed convergence~(4').} For $f = G^\omega$, every
$P$ reaches its fixed point $f(P)$ in $k = 1$ step of $f$-iteration:
$f^1(P) = f(P)$ and $f^2(P) = f(f(P)) = f(P)$. The deeper content of
(4')---that the underlying $G$-iteration terminates---is Step~1.
\qed
\end{proof}

\begin{remark}[Role of G3 in Lemma G]
\label{rem:G-scope}
(G3) is a guard-stability hypothesis, not the source of idempotence
and not the source of saturation well-definedness. Idempotence follows
from saturation (Step~5). Well-definedness of $G^\omega$ comes from
finite additive termination (Step~1, by R1 + finite
$\mathcal{C} \cdot \Dom$) plus (G2) functionality (Step~2). (G3)
records the additional property that suppression persists along the
trajectory: rules suppressed by positive conflict witnesses do not
become available later through additive updates. For the ICT guard,
(G3) holds automatically because M2 and M2-HTF are positive
forbidden-pair relations---predicate additions only increase conflict
witnesses, so suppression persists. (G3) is therefore a structural
property of the ICT-style guard that the framework records explicitly,
not a load-bearing hypothesis for the saturation theory.

Lemma~G does \textbf{not} claim monotonicity in $\Pspec$. As shown in
\S6.2 (Remark~\ref{rem:monotonicity-fails}), (3c)-realized mechanisms
with conflict-suppression guards are generically non-monotone in
$\Pspec$. The (4') conclusion above does not depend on monotonicity;
only extensivity, (G2), (G3), and Soundness are needed.
\end{remark}

\begin{remark}[(G2)--(G3) and Soundness for ICT $\DcompHTF$]
The guard mechanism for $D_L$ in \S6.2 satisfies (G2), (G3), and
Soundness:
(G2) is the Functionality Lemma (Lemma~\ref{lem:functionality}) of
\S6.2.
(G3) holds because M2 and M2-HTF are positive forbidden-pair relations:
predicate additions only \emph{increase} conflict witnesses. Once a
rule is suppressed by $c'$, $c' \in S$ at all subsequent states (by
R1), so suppression persists.
Soundness is given by Lemma~\ref{lem:6b} (single-step admissibility
preservation under guard).
Hence $\DcompHTF$ is in scope of Lemma~G, providing a uniform proof of
properties (1), (2), (3), (4'); the \S6.2 verification gives the same
conclusion via direct argument.
\end{remark}

\begin{remark}[Realization-coverage for U3, informal --- OQ-Realization]
\label{rem:realization-coverage}
Theorem~U3's case partition assumes that every valid theory-intrinsic
mechanism is realized by exactly one of (3a)/(3b)/(3c). This claim is
\emph{informal}: by (I1), $f$ is definable from $T$, so $f(P)$ depends
on $P$ through some combination of $I^{\mathrm{core}}$ applications and
selector queries. Pure $I^{\mathrm{core}}$ saturation under elementary
firing is (3a); pure selector from axiom-induced order is (3b);
$I^{\mathrm{core}}$ restricted by guards/scheduling derived from $A$,
in saturated form, is (3c). Mechanisms that mix selector outputs with
$I^{\mathrm{core}}$ rule conclusions in ways not reducible to
(3a)/(3b)/(3c) (for example: hybrids embedding selector lookups inside
$I^{\mathrm{core}}$ derivations, cardinality-minimizing operators,
lexicographic selectors induced by named predicates, fixed-point
selection via secondary invariants, quotient/orbit selectors under
automorphism groups, or priority-valued canonicalizers encoded directly
in $A$) are not obviously covered. The argument above is informal
because we have not formally classified all expressions definable from
$(\Sigma, A, I)$. Establishing this as a theorem is left as
\textbf{OQ-Realization}: whether (3a)/(3b)/(3c) is provably exhaustive
for all theory-intrinsic operations definable in MST. Theorem~U3's
conclusion is conditional on this assumption.
\end{remark}

\begin{remark}[Classification of $\DcompHTF$]
$\DcompHTF$ is a Class~C mechanism: it is (3c)-realized (saturated
guarded closure) and satisfies all four properties of \S4.1 (verified
in \S6.2; properties (1), (2), (3), (4') also covered by Lemma~G with
(G2)+(G3)+Soundness). U3 places it under Case~3.
\end{remark}

\subsection{Comparability Structures (Refined N4)}

\begin{definition}[Comparability structure]
$\mathcal{K} = (D, \preceq_{\mathcal{K}}, \mathcal{A}, \mathcal{C})$
where $D$ is a domain of admissible candidates,
$\preceq_{\mathcal{K}}$ an order on $D$, $\mathcal{A}$ a local
admissibility predicate, $\mathcal{C}$ a global compatibility
condition.
\end{definition}

\begin{definition}[N4 --- refined]
$T$ satisfies N4 if there exists a theory-intrinsic $\mathcal{K}$ such
that:
\begin{enumerate}
\item local sets $A_{r,t} \subseteq D$ are well-defined;
\item $\preceq_{\mathcal{K}}$ induces uniqueness on each $A_{r,t}$ via:
  \emph{selector-compatible} (N4-sel): unique maximum;
  \emph{closure-compatible} (N4-comp): extensive, idempotent,
  admissibility-preserving map with at least one fixed point;
  \emph{piecewise-compatible}: domain decomposes into subdomains, each
  admitting one of the above;
\item the induced mechanism extends to globally admissible assignments
  under $\mathcal{C}$.
\end{enumerate}
\end{definition}

\noindent\textit{Relationship to \S5.1.} The \S5.1 statement is a
working summary; \S8.3 is the structural form. Theorem~6b's branches
use N4-comp ($\equiv$ closure-compatible) and N4-sel ($\equiv$
selector-compatible).

\begin{remark}
$T_{\mathrm{Chan}}$ is closure-compatible (rewrite normal form, under
OQ-Chan-TRS); $T_{ICT}$ is closure-compatible with M2/M2-HTF
constraints ((4') strength; (4) conditional on OQ-GC-1); $T_{Wy}$ is
selector-compatible (event dominance order).
\end{remark}

\subsection{Determinization as a Pseudofunctor}

\begin{definition}[Det]
$\mathrm{Det}(T, \mathcal{K}) = T^{\mathrm{canon}}_{\mathcal{K}}$,
obtained by adjoining the uniqueness mechanism induced by
$\mathcal{K}$. Signature and axioms preserved; inference policy
extended:
$I^{\mathrm{canon}} = I^{\mathrm{core}} \cup
I^{\mathrm{det}}(\mathcal{K})$. This construction is mechanism-relative:
closure-based in Class~C (N4-comp), selector-based in Class~S (N4-sel).
\end{definition}

Category $\mathbf{MST}_{\mathcal{K}}$: objects are pairs
$(T, \mathcal{K})$ where $T$ is determinizable. Morphisms
$(\varphi, \eta)$: $\varphi: T_1 \rightharpoonup T_2$ a partial
morphism, $\eta$ preserves $\mathcal{K}$.
Det is a partial pseudofunctor
$\mathbf{MST}_{\mathcal{K}} \to \mathbf{MST}$: composition and
identities hold up to 2-morphism (domain inclusion).

\begin{remark}
Pseudofunctor coherence---that composition in
$\mathbf{MST}_{\mathcal{K}}$ corresponds (up to 2-morphism) to
composition under Det---is stated here. Detailed coherence witnesses
(Canon-1, Canon-2) are provided in the external artifact
\texttt{chan\_canon\_pseudofunctor\_2026-04-13.md}. A self-contained
in-paper proof is \textbf{OQ-Det-Coh}.
\end{remark}

\subsection{Status of \S8 Results}

\noindent\fbox{\begin{minipage}{0.97\linewidth}\small
\textbf{Theorem~6b (Construction Lemma):} ESTABLISHED by direct
construction; substantive content is Corollary~6b'.

\textbf{Corollary~6b' ((3a)-scope):} ESTABLISHED at~(4') strength via
Lemma~C + Monotonic Exhaustion (per-seed) + saturation, \textbf{under
the Soundness hypothesis (\S5.3) and elementary one-rule firing
interpretation of $F$}. Does not establish (4) or AC-6. (3a)
interpreted as saturated closure $F^\omega$; one-step operator $F$
alone is generally not idempotent.

\textbf{Lemma~R2*:} ESTABLISHED (parts i, ii); characterizes asymmetry,
not mechanism reduction (see U3).

\textbf{N4 (refined):} ESTABLISHED (\S8.3). Type~E: N4-comp;
Type~S-strong: N4-sel.

\textbf{Theorem~U3 (conditional classification):} ESTABLISHED as
classification \emph{within} (3a)/(3b)/(3c) realization grammar,
conditional on OQ-Realization \textbf{and on per-realization branch
hypotheses now lifted into the U3 statement} (R1+R2+R3 + Soundness +
elementary $F$ for (3a)/(3c); $\preceq_{\mathrm{sel}}$ + C5-S for (3b);
(G2) + (G3) for (3c)). Without those hypotheses, U3 does not establish
global no-third-class. Theory-intrinsic via (I1)--(I2). U3 is no longer
a tautology in $\mathrm{Class\ C} \cup \mathrm{Class\ S}$, but a
derivation of class membership from per-realization assumptions.

\textbf{Lemma~G ((3c)-scope):} ESTABLISHED for $f = G^\omega$ under
(G2) functionality, (G3) guard-stability (suppression persists along
the trajectory), and Soundness. Provides extensivity + admissibility
preservation + idempotence (by saturation) + (4') (with $k = 1$ for
$f$). Well-definedness of $G^\omega$ does not require G3 (only R1 +
finite + G2). NOT monotone in $\Pspec$ (counterexample \S6.2,
Remark~\ref{rem:monotonicity-fails}).

\textbf{$\DcompHTF$ classification:} (3c)-mechanism (saturated form).
Class~C by direct \S4.1 verification in \S6.2 and by Lemma~G with
(G2)+(G3)+Soundness. NOT monotone in $\Pspec$ (concrete witness in
\S6.2).

\medskip
\textbf{Open questions:}
\begin{itemize}
\item OQ-GC-1: whether ICT~(4') upgrades to global~(4).
\item OQ-TypeS-Imp: whether non-strict Type~S theories admit~(4)-strength
  completion.
\item OQ-Chan-TRS: formal Chan TRS termination + local confluence.
\item OQ-Det-Coh: in-paper proof of Det pseudofunctor coherence.
\item OQ-Realization: whether (3a)/(3b)/(3c) is provably exhaustive for
  theory-intrinsic operations.
\end{itemize}

\medskip
\textbf{External dependencies (non-OQ):}
\begin{itemize}
\item Phase~2 results (\S10.2): summary statement; full proofs in
  companion artifacts \texttt{phase2b} and \texttt{phase2c}, not
  reproduced in this paper.
\end{itemize}

\medskip
\textbf{ICT canonicalization status:}
\begin{itemize}
\item Closure stabilization (4'): ESTABLISHED (\S6.2).
\item Global completion (4): CONDITIONAL on OQ-GC-1.
\item Determinization (AC-6): CONDITIONAL on OQ-GC-1.
\end{itemize}

\textbf{Type~S-strong $\to$ selection (AC-6):} ESTABLISHED
(Proposition~\ref{prop:completion-impossible}, \S4.2). Generic Type~S
$\to$ selection: NOT proven (OQ-TypeS-Imp).

\textbf{Non-commutativity:} Witness ESTABLISHED (\S7
Proposition~\ref{prop:noncomm}). Ordering uniqueness scoped to
two-operator raw staged architecture
(Corollary~\ref{cor:noncomm-uniqueness}); uniqueness across all
multi-stage architectures is NOT claimed.
\end{minipage}}

\section{Discussion}

\textbf{Stabilization vs.\ determinization.} A key structural distinction
is between closure stabilization (per-seed convergence, property~(4'))
and full determinization (global unique fixed point + AC-6). The
current framework fully characterizes closure stabilization (Class~C)
and selection-based canonicalization (Class~S); the upgrade from
stabilization to determinization is governed by additional global
confluence properties that remain to be characterized. For ICT, this
is OQ-GC-1.

\textbf{Mechanism classification.} \S8 establishes that, conditional on
the realization-coverage assumption (OQ-Realization) and on per-branch
hypotheses, no third primitive canonicalization mechanism arises within
the (3a)/(3b)/(3c) realization grammar. The proof works in two layers:
theory-intrinsic operations are defined independently via (I1)--(I2);
the trichotomy (3a)/(3b)/(3c) is the realization case partition;
Theorem~U3 then classifies each case using Lemma~C ((3a)-scope
monotonicity from R1+R2+R3) and Lemma~G ((3c)-scope extensivity +
admissibility preservation + (4') + idempotence). This two-layer
structure replaces the v2.15.1 formulation in which the trichotomy was
embedded in the theory-intrinsic definition (definitional circularity).
Realization-coverage is tracked as OQ-Realization, and U3 is explicitly
conditional on it. (3a)-scope Class~C derivation requires both syntactic
R1+R2+R3 and the semantic Soundness hypothesis; the two are independent.

\textbf{Monotonicity in $\Pspec$ is mechanism-specific.}
Pure $I^{\mathrm{core}}$-computable mechanisms (3a, in saturated form
under elementary firing) are automatically monotone in $\Pspec$ under
R1+R2+R3 (Lemma~C). Guarded mechanisms (3c) are generically \emph{not}
monotone: adding predicates may activate previously-inactive conflict
guards, blocking derivations that fire in the smaller state. The \S6.2
explicit witness for $\DcompHTF$ (Remark~\ref{rem:monotonicity-fails})
makes this concrete. This is not a defect: Lemma~G's (4') conclusion
does not require monotonicity, and the Monotonic Exhaustion Lemma uses
only finite + extensive. The asymmetry between (3a)- and (3c)-scope
monotonicity is a structural feature of guard-based determinization,
reflecting the architectural separation between $I^{\mathrm{core}}$
derivation (monotone by R1+R2+R3) and conflict resolution (not monotone
in general).

\textbf{Type~S, Type~S-strong, and the scope of selection.} The
dichotomy ``Type~E $\leftrightarrow$ completion, Type~S
$\leftrightarrow$ selection'' is sound only with the Type~S-strong
qualifier. Type~S theories without the no-common-upper-bound
strengthening may admit completion at~(4) strength depending on
upper-bound structure (OQ-TypeS-Imp). $T_{Wy}$ is Type~S-strong by
axiomatic phase incomparability, justifying the selection construction.
The framework's substantive claim is therefore narrower than a generic
Type~S $\to$ selection inference: Type~S-strong provably blocks
(4)-strength completion, while leaving the broader question open.
Crucially, the selection mechanism in Type~S-strong precisely operates
over alternatives that lack a common upper bound, so the C5-S condition
asserts global admissibility of the selector's output, not joint
admissibility of local alternatives---these are distinct claims.

\textbf{Obstacle to global completion in ICT.} The obstruction to
upgrading ICT from~(4') to~(4) is temporal: directional alternatives
(e.g., BOS direction) are not yet decided at a given input but become
decided under further evidence refinement. Different seeds may
stabilize at different fixed points precisely because the evidence has
not yet forced unique directional commitment. Hence OQ-GC-1 remains
open.

\textbf{The role of C5.} C5 captures the requirement that
canonicalization output is globally consistent. The two variants
operate on different mathematical objects: C5-E asserts compatible
extension within the closure domain (a set-valued condition); C5-S
asserts global admissibility of the selector's output (a forward
condition on the selected assignment). The variants share the
intuition of canonicalization output coherence but apply to operators
versus selectors respectively.

\textbf{Non-commutativity and hierarchy.} \S7 establishes that, within
the two-operator raw staged architecture, multi-level canonicalization
is not reducible to iterated single-level canonicalization with
arbitrary ordering. The HTF-first ordering is a structural property
within this declared architecture; alternative architectures (with
repair operators, joint fixed-point constructions, or symmetrically
guarded $D_H$) are not excluded but lie outside the present scope.

\textbf{Toward recursive theories.} Elliott Wave and similar
hierarchical theories require an ordered family $\{D_\ell\}_{\ell \in
L}$ with level-indexed M2-HTF guards and a compatibility condition
across levels. The Det pseudofunctor of \S8.4 provides the categorical
template.

\textbf{Connection to LLM-assisted reasoning.} The framework's
conceptual payoff for AI systems is the formalization of when an
LLM-assisted system is licensed to commit to a unique structural
interpretation. In this view, hallucination corresponds to forced
canonicalization at AC-6 strength when only stabilization (4') is
licensed, or at any strength when neither N4 nor C5 holds. Type~E
canonicalization (closure stabilization) is the appropriate semantics
for systems where evidence accumulates over time; Type~S-strong
canonicalization (selection) is appropriate when alternatives are
axiomatically incomparable. Distinguishing these regimes is
prerequisite for a principled treatment of structural commitment in
LLM systems.

\section{Phase~2: Multi-Operator Canonicalization}

\emph{All definitions in \S10.1 are self-contained. The existence and
hierarchy results in \S10.2 are summary statements; full proofs are
recorded in the companion artifacts \texttt{phase2b} and \texttt{phase2c}
and are not reproduced in this paper. The present section serves as a
unified reference point and as scaffolding for the Det pseudofunctor of
\S8.4.}

\subsection{Setting and Definitions}

\begin{definition}[DQA --- Determinization Quasi-Algebra]
$\mathcal{Q} = (\mathfrak{D}, \circ, \mathcal{A}, \prec, \mathcal{O})$:
finite family $\mathfrak{D} = \{D_\ell\}_{\ell \in L}$, partial
composition $\circ$, admissible sequences $\mathcal{A}$, precedence
poset $\prec$, obstructed sequences $\mathcal{O}$.

Each $D_\ell$ satisfies:
\begin{description}
\item[(Q1) Local closure] Extensivity, idempotence, admissibility
  preservation, per-seed convergence~(4') on its standard domain.
\item[(Q2) Partial composition] Defined only on admissible sequences.
\item[(Q3) Obstruction]
  $\sigma \in \mathcal{O} \iff \exists\, P \in \mathcal{P}_T:
  \sigma(P) \notin \mathcal{P}_T$.
\item[(Q4) Idempotent stability (partial-function form)]
  $D_\ell \circ D_\ell$ and $D_\ell$ agree as partial functions: for
  all $P \in \mathcal{P}_T$ on which both sides are defined,
  $(D_\ell \circ D_\ell)(P) = D_\ell(P)$.
\end{description}
\end{definition}

\begin{remark}[Q4 partial-function form]
Q4 is not redundant with Q1 pointwise idempotence. Q1 is per-state on
the standard domain; Q4 is operator-level on the domain where both
sides are defined. The partial-function form makes the domain
explicit. In the single-operator setting on the standard domain, Q1
and Q4 coincide; in a DQA where $D_\ell$ may be applied in composition
contexts, Q4 provides the additional stability guarantee restricted to
the admissible domain.
\end{remark}

\begin{definition}[Standard domain]
The \emph{standard domain} of $D_\ell$ is the canonical seed domain on
which $D_\ell$ is verified in isolation to satisfy Q1's four
properties. For ICT, the standard domain of $\Dcomp$ is
$\PICT^{\mathrm{closure}}$; for $\DcompHTF$ it is
$\PICT^{\mathrm{multi}}$.
\end{definition}

\begin{definition}[Safe application domain]
$\Safe(D_\ell) := \{P \in \mathcal{P}_T \mid D_\ell(P) \in
\mathcal{P}_T\}$.
\end{definition}

\begin{remark}[Q1 vs.\ Safe]
Q1 guarantees admissibility preservation on the standard domain. In a
multi-operator system, $D_\ell$ may be applied to states produced by
preceding operators that fall outside the standard domain.
$\Safe(D_\ell)$ captures the actual safe application domain in the
composition context; the safety gap $\Phi^*$ measures the obstruction.
\end{remark}

\begin{definition}[Prefix-reachable states]
For topological ordering $\ell_1, \ldots, \ell_n$ and seed set
$\mathcal{S}$:
$\Reach_0(\mathcal{S}) := \mathcal{S}$;
$\Reach_{k+1}(\mathcal{S}) := \{D_{\ell_{k+1}}(Q) \mid Q \in
\Reach_k(\mathcal{S})\}$.
\end{definition}

\begin{definition}[Admissible linearization]
A topological ordering $\ell_1, \ldots, \ell_n$ of $(L, \prec)$ is
\emph{admissible on $\mathcal{S}$} if the composed operator
$(D_{\ell_n} \circ \cdots \circ D_{\ell_1})(P) \in \mathcal{P}_T$ for
all $P \in \mathcal{S}$. Equivalently,
$\Reach_k(\mathcal{S}) \subseteq \Safe(D_{\ell_{k+1}})$ for all $k < n$.
\end{definition}

\begin{definition}[Safety gap]
\[
  \Phi(\ell_1,\ldots,\ell_n;\mathcal{S}) :=
  \sum_{k=0}^{n-1}
  \mathbf{1}\{\Reach_k(\mathcal{S}) \nsubseteq \Safe(D_{\ell_{k+1}})\}
\]
$\Phi^*(\mathcal{S}) := \min_{\text{topological}}
\Phi(\cdot;\mathcal{S})$.
$\Phi^*(\mathcal{S}) = 0 \iff \exists$ admissible linearization on
$\mathcal{S}$.
\end{definition}

\noindent\textbf{Compatibility conditions.}

\textbf{(GCC) Graph-Compatible Guard Condition} (seed-independent):
$\mathrm{Im}(D_{\ell_k}) \cap \mathcal{P}_T \subseteq
\Safe(D_{\ell_{k+1}})$ for all consecutive $\ell_k, \ell_{k+1}$.

\textbf{($\mathrm{RUC}_{\mathcal{S}}$) Reach-Uniform Compatibility}
(seed-dependent): for every extendable prefix $U_k$ and every
$\ell, \ell' \in \Next(U_k)$:
$\Safe(D_\ell) \cap \Reach_k(\mathcal{S}) =
\Safe(D_{\ell'}) \cap \Reach_k(\mathcal{S})$.

\textbf{($\mathrm{PLC}_{\mathcal{S}}$) Pairwise Local Compatibility}
(seed-dependent): for every extendable prefix $U_k$ and every
$\ell, \ell' \in \Next(U_k)$:
$\Reach_k(\mathcal{S}) \subseteq \Safe(D_\ell) \Leftrightarrow
\Reach_k(\mathcal{S}) \subseteq \Safe(D_{\ell'})$.

\noindent\textbf{Extension conditions.}
\textbf{(C2)} $\exists\, \ell \in \Next(U_k)$ with
$\Reach_k(\mathcal{S}) \subseteq \Safe(D_\ell)$ at each prefix.
\textbf{(C2')} $\forall\, \ell \in \Next(U_k)$,
$\Reach_k(\mathcal{S}) \subseteq \Safe(D_\ell)$ at each prefix.

\subsection{Phase~2 Summary Theorem}

\begin{theorem}[Phase~2 summary; proofs in companion artifacts]
\label{thm:phase2}
\textbf{(Statement; full proofs are not reproduced in this paper. They
are recorded in the companion artifacts \texttt{phase2b} and
\texttt{phase2c}. The present theorem serves as a unified reference
point and as scaffolding for the Det pseudofunctor of \S8.4.)}

Let $\mathcal{Q}$ be a DQA with finite $L$ and $\mathcal{S} \subseteq
\mathcal{P}_T$.

\emph{(I) Existence.} (Requires $(L,\prec)$ acyclic.)
\[
  \Phi^*(\mathcal{S}) = 0 \Longleftrightarrow
  \exists\,\text{admissible linearization on }\mathcal{S}.
\]
Under $\mathrm{PLC}_{\mathcal{S}}$:
$\Phi^* = 0 \Leftrightarrow \mathrm{(C2)} \Leftrightarrow
\mathrm{(C2')} \Leftrightarrow \exists$ admissible linearization.

\emph{(II) Compatibility-strength hierarchy.} The compatibility
conditions are ordered by logical implication strength (not by set
inclusion):
\[
  \mathrm{GCC} \Rightarrow \mathrm{RUC}_{\mathcal{S}} \Rightarrow
  \mathrm{PLC}_{\mathcal{S}}
  \quad\text{(seed-dependent, both implications strict)}
\]
\[
  \mathrm{GCC} \Rightarrow \mathrm{RUC} \equiv \mathrm{PLC}
  \quad\text{(seed-independent)}
\]
\end{theorem}

\begin{proof}[Proof sketch]
(I): Theorems~2b-14/2b-15 of phase2b artifact; collapse under
$\mathrm{PLC}_{\mathcal{S}}$ by Theorem~2c-1 of phase2c artifact.
(II): Propositions~2c-2 through~2c-6 of phase2c artifact. Detailed
proofs are not reproduced here; the present section is a summary.
\qed
\end{proof}

\begin{remark}[Notational convention]
We use $\mathrm{GCC} \Rightarrow \mathrm{RUC}_{\mathcal{S}}$ to denote
that GCC is logically stronger than $\mathrm{RUC}_{\mathcal{S}}$ (a
DQA satisfying GCC also satisfies $\mathrm{RUC}_{\mathcal{S}}$). This
is implication ordering, distinct from set inclusion. ``Strict'' means
the implication is not biconditional: there exist DQAs satisfying
$\mathrm{RUC}_{\mathcal{S}}$ but not GCC, and DQAs satisfying
$\mathrm{PLC}_{\mathcal{S}}$ but not $\mathrm{RUC}_{\mathcal{S}}$.
\end{remark}

\subsection{Strictness Witness}

\begin{example}[$\mathrm{PLC}_{\mathcal{S}}$ without
$\mathrm{RUC}_{\mathcal{S}}$]
$L = \{a,b\}$, $\prec = \varnothing$, $\mathcal{P}_T = \{Q_1, Q_2,
Q_3\}$, $\mathcal{S} = \mathcal{P}_T$.
$\Safe(D_a) \cap \Reach_0 = \{Q_1, Q_2\}$;
$\Safe(D_b) \cap \Reach_0 = \{Q_1, Q_3\}$.
Both $\chi_a = \chi_b = 0$ ($\mathrm{PLC}_{\mathcal{S}}$ holds);
restricted safe sets differ ($\mathrm{RUC}_{\mathcal{S}}$ fails). PLC
checks the binary all-or-nothing judgment; RUC checks set identity.
\end{example}

\section{Changelog}

This section records changes across the v2.15.1 $\to$ v2.16 $\to$
v2.16.2 $\to$ v2.16.3 $\to$ v2.16.4 revision sequence. v2.16 was an
errata revision of v2.15.1 addressing 13 issues; v2.16.2 was an
incremental fix addressing one identified gap (M1, Lemma~G monotonicity)
plus minor follow-up; v2.16.3 tightens four claim-strength items
identified in post-v2.16.2 review; v2.16.4 is a mathematical-correctness
patch addressing four findings from a third-round adversary review (S1
Soundness, S3 Lemma~G saturated form with G3 reframed as guard-stability,
S4 Phase~2 summary qualifier, U3 statement-level de-tautology).

\subsection{v2.16.3 $\to$ v2.16.4 (mathematical correctness fixes)}

\textbf{S1 (Soundness assumption added to Corollary 6b').}
v2.16.3 \S5.3 Corollary~6b' implicitly assumed admissibility
preservation under the phrase ``by construction of
$I^{\mathrm{core}}$-computability'' but did not declare it as a
hypothesis. The reviewer's reflexive counterexample (rule
``$a \Rightarrow b$'' plus axiom forbidding $\{a, b\}$ satisfies
R1+R2+R3 syntactically while saturated closure exits $\mathcal{P}_T$)
shows the corollary as stated is false. v2.16.4
(i)~reintroduces $f$ as $F^\omega$ on the unrestricted state space and
uses Soundness to derive the restriction $\mathcal{P}_T \to
\mathcal{P}_T$ (avoiding type-circularity of writing
``$f: \mathcal{P}_T \to \mathcal{P}_T$'' and then proving it),
(ii)~declares Soundness as an explicit additional hypothesis,
(iii)~updates the proof body to invoke Soundness for admissibility
preservation,
(iv)~adds Remark~\ref{rem:soundness-independent} distinguishing the
syntactic constraints R1+R2+R3 from the semantic Soundness condition,
and (v)~makes the elementary one-rule-firing interpretation of $F$
explicit at statement level, so single-rule Soundness suffices (a
separate counterexample exists for simultaneous-firing semantics: rules
``$a \Rightarrow b$'' and ``$a \Rightarrow c$'' plus axiom forbidding
$\{b, c\}$ are individually sound but jointly unsound). The proof also
adds an order-independence sentence noting that, in the pure (3a)
positive-rule setting under fair elementary scheduling, saturation is
order-independent (every conclusion whose positive premises eventually
become true is eventually added; no rule retracts). For $T_{ICT}$ the
soundness role is filled by M1.

\smallskip\textbf{S3 (Lemma G saturated form, mirror of T4; G3 reframed
as guard-stability).} v2.16.3 T4 fixed the (3a) realization as the
saturated closure $f = F^\omega$. v2.16.3 Lemma~G ((3c)-scope) retained
the ambiguity between one-step operator and saturation: the proof
showed iteration reaches a fixed point $P^*$ and concluded
$f(P^*) = P^*$, but did not establish $f \circ f = f$ for arbitrary
$P$. v2.16.4
(i)~declares $f := G^\omega$ explicitly (one-step guarded operator $G$,
saturated closure $G^\omega$), making idempotence a saturation
tautology,
(ii)~rewrites the proof body as a 6-step argument
($G$-trajectory termination $\to$ $G^\omega$ well-defined $\to$
restriction by Soundness $\to$ extensivity $\to$ idempotence by
saturation $\to$ (4') with $k = 1$ for $f$),
(iii)~\textbf{reframes G3 as a guard-stability property of the
ICT-style guard, not as the source of saturation well-definedness ---
well-definedness of $G^\omega$ comes from R1 + finite + (G2)
functionality, while (G3) ensures that suppression persists along the
trajectory},
(iv)~synchronizes the \S8.0 (3c) definition with the saturated form
($f = G^\omega$ explicit),
and (v)~adds the same Soundness hypothesis as S1, applied to $G$.
Internal consistency with the (3a) treatment under T4.

\smallskip\textbf{S4 (Phase 2 theorem qualifier).} v2.16.3 \S10.2 stated
a theorem whose detailed proofs reside in companion artifacts
\texttt{phase2b} / \texttt{phase2c}. v2.16.4
(i)~extends the theorem name with ``Phase 2 summary; proofs in
companion artifacts'',
(ii)~prepends a bold parenthetical clarifying that proofs are not
reproduced in the paper and recording the companion artifacts,
and (iii)~extends the \S10 opener accordingly. Wording is restrained
(no claim that the artifacts constitute ``established external
results''), reflecting that the artifacts are not necessarily publicly
accessible.

\smallskip\textbf{U3 statement-level de-tautology.} v2.16.3 Theorem~U3
had a hypothesis of the form ``$f$ is a valid canonicalization
mechanism, i.e.\ Class~C or Class~S''---putting the conclusion inside
the hypothesis. v2.16.4 deletes this clause and lifts the
per-realization assumptions into the statement: R1+R2+R3 + Soundness +
elementary $F$ for (3a); theory-intrinsic $\preceq_{\mathrm{sel}}$ +
C5-S for (3b); R1+R2+R3 + (G2) + (G3) + Soundness for (3c). The
Class~C / Class~S conclusion is now derived per-case from per-branch
hypotheses, not pre-assumed.

\smallskip
No statement of any other theorem, lemma, or corollary changes; no
classification result changes scope. v2.16.4 is a mathematical-correctness
patch, not a content revision.

\subsection{v2.16.2 $\to$ v2.16.3 (claim-strength tightening)}

\textbf{T1 (C5-S redefined: global compatibility of selected choices).}
v2.16.2 \S5.1 defined C5-S as ``C5 applies over all locally admissible
pairs $c_1, c_2 \in A_{r,t}$''---i.e., requiring joint admissibility of
local alternatives. This conflicts with Type~S-strong, which is
defined precisely by alternatives lacking a common upper bound.
v2.16.3 redefines C5-S as global admissibility of the selector's
output: the assignment $P^*: (r,t) \mapsto \{c^*_{r,t}\}$ extends to an
element of $\mathcal{P}_T$, where $c^*_{r,t} =
\max_{\preceq_{\mathrm{sel}}} A_{r,t}$. Selection in Type~S-strong now
operates correctly: the selector picks one alternative per locus, and
C5-S asserts that the resulting global assignment is admissible. C5-E
(completion variant) is unchanged.

\smallskip\textbf{T2 (Theorem~U3 explicitly conditional on
OQ-Realization).} v2.16.2 stated U3 as ``no third primitive class
arises in MST'' with realization-coverage tracked as informal. v2.16.3
restates U3 as conditional on OQ-Realization: ``\emph{Assume the
realization-coverage hypothesis (OQ-Realization).} Then within the
(3a)/(3b)/(3c) realization grammar, every valid theory-intrinsic
canonicalization mechanism is Class~C or Class~S.'' Theorem title
changed to ``Conditional Classification within the MST Realization
Grammar.'' Realization-coverage examples expanded in
Remark~\ref{rem:realization-coverage} to flag mechanisms outside the
declared grammar (cardinality-minimization, orbit selectors,
priority-valued canonicalizers, hybrid selector/closure operators).

\smallskip\textbf{T3 (Non-commutativity ordering uniqueness scoped).}
v2.16.2 stated ``HTF-first ordering is the unique
admissibility-preserving staged completion order'' without scope
qualifier. v2.16.3 \S7 separates the witness (Proposition, unchanged)
from a scoped ordering-uniqueness corollary
(Corollary~\ref{cor:noncomm-uniqueness}): ``\emph{Within the
two-operator raw staged architecture
$\{D_L \circ D_H, D_H \circ D_L\}$, with no repair operator, no M2-HTF
guard added to $D_H$, and no joint fixed-point construction}, HTF-first
is the only admissibility-preserving linearization.'' A scope remark
explicitly notes that alternative architectures (with repair operators,
symmetrically guarded $D_H$, or joint fixed-point constructions) lie
outside the present scope.

\smallskip\textbf{T4 ((3a) saturated form made explicit).}
v2.16.2 referred to ``$I^{\mathrm{core}}$-computable in the pure sense
(3a)'' without distinguishing one-step rule application from saturated
closure. v2.16.3 makes the saturated interpretation
$f = F^\omega$ explicit in \S8.0 (3a) definition,
Corollary~6b' statement and proof, Lemma~C statement, U3 Case~1, and
\S8.5 status board. The one-step operator $F$ alone is generally not
idempotent; saturation is required for idempotence and is built into
the (3a) realization throughout.

\smallskip\textbf{T5 (LLM-assisted reasoning hook restored).} v2.16.2
abstract removed the LLM-motivation framing present in v2.15.1.
v2.16.3 restores an opening LLM/cs.AI hook in the abstract and adds an
explicit \S1 introduction paragraph and a closing \S9 paragraph
connecting canonicalization to hallucination as ``unsupported
canonicalization.''

\smallskip\textbf{T6 (Phase~2 hierarchy notation: implication, not
inclusion).} v2.16.2 wrote $\mathrm{GCC} \supsetneq
\mathrm{RUC}_{\mathcal{S}} \supsetneq \mathrm{PLC}_{\mathcal{S}}$,
which is set-inclusion notation but in the wrong direction (GCC is
the strongest condition, so the set of GCC-satisfying DQAs is the
\emph{smallest}). v2.16.3 changes this to implication notation:
$\mathrm{GCC} \Rightarrow \mathrm{RUC}_{\mathcal{S}} \Rightarrow
\mathrm{PLC}_{\mathcal{S}}$. A notational convention remark is added
to \S10.2.

\smallskip\textbf{T7 (Core principle restated to align with C5-S/C5-E
asymmetry).} v2.16.2 abstract closed with ``Canonicalization is not a
consequence of closure alone, but of closure together with
comparability and joint admissibility.'' This treated C5-S and C5-E as
parallel, conflicting with T1. v2.16.3 restates: ``Completion-based
canonicalization requires closure, comparability, and compatible
extension; selection-based canonicalization requires theory-intrinsic
comparability and global compatibility of selected choices.''

The (4') closure-stabilization conclusion is unchanged in scope and
strength. All Class~C/Class~S classifications, the Type~S-strong
impossibility result, and ICT/Wyckoff/Chan instance assignments are
unchanged.

\subsection{v2.16 $\to$ v2.16.2 (M1 fix and minor follow-up)}

\textbf{M1 (Lemma~G monotonicity claim removed).} v2.16 \S8.2 Lemma~G
asserted that (3c)-realized mechanisms inherit monotonicity in
$\Pspec$ from $I^{\mathrm{core}}$ rules under hypotheses (G1)
state-monotone guard, (G2) functionality, (G3) additive-stable guard.
Audit identified that (G1) does not hold in the asserted direction for
ICT $\DcompHTF$: guard suppression on $Q$ does not imply guard
suppression on $P$ when $P \Pspec Q$---the opposite direction is what
holds (more predicates means more witnesses, hence more suppression).
This made the v2.16 monotonicity claim incorrect in its application to
ICT.

\smallskip A concrete witness shows $\DcompHTF$ is not monotone in
$\Pspec$: take $P = \emptyset$ and $Q = \{\mathrm{HTF\_bearish\_bias}\}$
at an input with $\mathrm{close} > \mathrm{prev\_close}$. Then $\rho_3$
fires on $P$ (yielding $\mathrm{LTF\_bullish\_cont} \in \DcompHTF(P)$)
but is M2-HTF-suppressed on $Q$ (so $\mathrm{LTF\_bullish\_cont} \notin
\DcompHTF(Q)$).

\smallskip v2.16.2 fixes this as follows:
\begin{itemize}
\item \S8.2 Lemma~G: removed monotonicity claim and (G1) hypothesis.
  Lemma~G now establishes only extensivity, per-seed~(4'), and
  idempotence under (G2)+(G3). Sufficient because Monotonic Exhaustion
  already requires only finite + extensive (monotonicity was
  redundant).
\item \S8.2 Monotonic Exhaustion Lemma: hypothesis tightened to
  ``extensive + finite''.
\item \S8.2 Lemma~C: added Remark~\ref{rem:lemmaC-nonextension}
  ``Lemma~C does not extend to (3c) under naive $\Pspec$''.
\item \S6.2: added Remark~\ref{rem:monotonicity-fails} with the
  explicit ICT counterexample.
\item \S8.5 status board: updated Lemma~G entry and $\DcompHTF$
  classification entry.
\item \S9 Discussion: added paragraph ``Monotonicity in $\Pspec$ is
  mechanism-specific''.
\end{itemize}

\smallskip\textbf{m4 (Realization-coverage marked informal;
OQ-Realization added).} v2.16 \S8.2 Realization-coverage remark argued
informally that every theory-intrinsic operation is realized by
(3a)/(3b)/(3c). v2.16.2 explicitly tagged this as informal and added
OQ-Realization to \S8.5 status board. (v2.16.3 further escalates this
to a conditional hypothesis on Theorem~U3; see T2.)

\smallskip\textbf{c2 (\S9 Discussion reordered).} v2.16.2 moved
``Stabilization vs.\ determinization'' to the first paragraph of \S9,
followed by ``Mechanism classification'' and ``Monotonicity in
$\Pspec$ is mechanism-specific''.

\subsection{v2.15.1 $\to$ v2.16 (errata revision, 13 issues)}

This subsection records the v2.15.1 $\to$ v2.16 changes. Changes are
grouped by issue.

\textbf{P1 (Theorem~U3 Case~3 tautology, framing).} \S8.2 Theorem~U3
renamed ``Primitive Mechanism Classification'' (was ``Primitive
Mechanism Exhaustion''). Hypothesis ``uniqueness-inducing'' removed
from U3 statement. Added scope note: U3 does not derive (4) or AC-6.
Added Lemma~G (Guard Preservation) in \S8.2 to give substantive proof
for Case~3 mechanisms; previous tautology replaced by explicit
(G2)+(G3) hypothesis (v2.16.2: Lemma~G scope restricted; (G1)
removed---see M1; v2.16.3: U3 made explicitly conditional---see T2;
v2.16.4: per-branch hypotheses lifted into U3 statement,
G3 reframed as guard-stability).
Added Realization-coverage remark (v2.16.2: marked informal,
OQ-Realization added).

\textbf{P2 (R2*(iii) restatement of U3).} \S8.1 Lemma~R2* part (iii)
deleted. R2* now consists only of parts (i) and (ii). Mechanism
reduction is the content of Theorem~U3, not R2*.

\textbf{P3 (C5-E vague definition).} \S5.1 C5-E rewritten as
framework-level explicit definition with the closure-domain conflict
relation made abstract. ICT instance (M2 / $\mathrm{M2} \cup
\mathrm{M2\text{-}HTF}$) given as example, not embedded in definition.

\textbf{P4 (Type~S $\to$ selection dichotomy overclaim).} Introduced
Type~S-strong as named subclass in \S3.1. \S4.2 Proposition renamed
and re-hypothesized for Type~S-strong. Wyckoff explicitly classified
as Type~S-strong. Abstract, \S1 contributions list, \S3 examples
table, \S4.2 architecture remark, \S8.5 status board, \S9 Discussion
all updated. OQ-TypeS-Imp tracks the open generic-Type-S case.

\textbf{P5 (uniqueness-inducing terminology + Monotonic Exhaustion
scope).} \S4.1: added Definition (Uniqueness-inducing operator) as
named alias for property (4). \S8.2 Monotonic Exhaustion Lemma renamed
``(per-seed)'', proof shrunk to use only finite + extensive $\to$ (4');
counterexample ($\{a,b\}$ + identity) inserted showing
(4') $\not\Rightarrow$ (4). \S5.3 Corollary~6b' hypothesis dropped
``uniqueness-inducing''. \S8.0 Class~C definition cleaned. \S8.1 R2*
hypothesis explicitly references \S4.1. \S8.2 U3 hypothesis dropped
``uniqueness-inducing''.

\textbf{P6 (Q4 domain ambiguity).} \S10.1 Q4 reformulated as
partial-function form: agreement on $P$ where both sides are defined.

\textbf{P7 (\S6.2 wording --- ``unique normal form for that seed'').}
Replaced by ``each seed converges to a fixed point within
$\PICT^{\mathrm{multi}}$''.

\textbf{P10 (\S5.3 Theorem~6b title).} Title changed to ``Construction
Lemma for Canonicalization Mechanisms''.

\textbf{R1 (Theorem~6b vs.\ Corollary~6b' position).} Theorem~6b
explicitly labeled ``(Construction Lemma)''; remark added explaining
``forward'' arrows are constructions. Corollary~6b' explicitly labeled
as ``the substantive structural result''.

\textbf{R2 (theory-intrinsic definition circularity).} \S8.0
theory-intrinsic definition rewritten using independent conditions
(I1)--(I2). The (3a)/(3b)/(3c) trichotomy is now a separate
``Realizations'' structure, the case partition for Theorem~U3.

\textbf{R3 (U3 Case~3 still circular).} Addressed via Lemma~G (\S8.2):
substantive proof that (3c)-realized mechanisms inherit (1), (2), (4')
properties under (G2)+(G3) (v2.16.2 scope; see M1; v2.16.4: G3
reframed as guard-stability, Soundness added---see S3).

\textbf{R4 (BOS premise mutual exclusion in \S6.2 Functionality
Lemma).} \S6.2: added ``Premise mutual exclusion'' subsection.
Functionality Lemma proof restructured.

\textbf{R5 (missing Lemmas 1--3 in \S6.2 idempotence).} \S6.2
idempotence verification rewritten as direct argument.

\textbf{R6 (\S5.1 N4 vs.\ \S8.3 refined N4).} \S5.1 N4 amended with
``Refined treatment'' subsection that explicitly references \S8.3.

\textbf{R7 (abstract honesty re ICT determinization).} Abstract
revised to explicitly state ``the present paper does not establish
ICT determinization at AC-6 strength''. ICT AC-6 conditional on
OQ-GC-1; Wyckoff AC-6 directly via PhaseClassify.

\textbf{R8 (Lemma~C two-class premise reasoning).} \S8.2 Lemma~C
proof expanded with explicit two-class analysis. R2 also clarified.

\textbf{R11 / OQ-Chan-TRS, OQ-Det-Coh tracking.} \S3.2 Chan rewrite
presentation tagged OQ-Chan-TRS. \S8.4 Det pseudofunctor coherence
tagged OQ-Det-Coh. \S8.5 status board updated with all OQs.

\appendix
\section{Conflict Matrices}

\subsection{M2 --- ICT Single-Timeframe Closure Domain}

\emph{Scope}: $\PICT^{\mathrm{closure}}$.
$\times$ = forbidden pair; $\checkmark$ = compatible.

\begin{center}
\small
\begin{tabular}{lcccc}
\toprule
& $\mathrm{BOS_{up}}$ & $\mathrm{BOS_{down}}$
& $\mathrm{bull\_cont}$ & $\mathrm{bear\_cont}$ \\
\midrule
$\mathrm{BOS_{up}}$    & --- & $\times$ & $\checkmark$ & $\times$ \\
$\mathrm{BOS_{down}}$  & $\times$ & --- & $\times$ & $\checkmark$ \\
$\mathrm{bull\_cont}$  & $\checkmark$ & $\times$ & --- & $\times$ \\
$\mathrm{bear\_cont}$  & $\times$ & $\checkmark$ & $\times$ & --- \\
\bottomrule
\end{tabular}
\end{center}

\noindent (Full predicate names: LTF\_BOS$_{\mathrm{up/down}}$,
LTF\_bullish/bearish\_cont.)

\subsection{M2-HTF --- Cross-Timeframe Forbidden Pairs}

\emph{Scope}: $\PICT^{\mathrm{multi}}$. HTF predicates dominate
opposing LTF predicates.

\begin{center}
\small
\begin{tabular}{lcccc}
\toprule
& $\mathrm{BOS_{up}}$ & $\mathrm{BOS_{down}}$
& $\mathrm{bull\_cont}$ & $\mathrm{bear\_cont}$ \\
\midrule
$\mathrm{HTF\_BOS_{up}}$       & $\checkmark$ & $\times$ & $\checkmark$ & $\checkmark$ \\
$\mathrm{HTF\_BOS_{down}}$     & $\times$ & $\checkmark$ & $\checkmark$ & $\checkmark$ \\
$\mathrm{HTF\_bull\_bias}$     & $\checkmark$ & $\checkmark$ & $\checkmark$ & $\times$ \\
$\mathrm{HTF\_bear\_bias}$     & $\checkmark$ & $\checkmark$ & $\times$ & $\checkmark$ \\
\bottomrule
\end{tabular}
\end{center}

\end{document}